\documentclass[twoside]{article}

\usepackage[preprint]{aistats2027}
\usepackage[round]{natbib}

\usepackage{multirow}
\usepackage{booktabs}  
\usepackage{graphicx}  
\usepackage{booktabs}
\usepackage{makecell}
\usepackage{tabularx}
\usepackage[table]{xcolor}
\usepackage{tablefootnote}

\definecolor{neurocolor}{RGB}{235,242,250}
\definecolor{chemcolor}{RGB}{245,239,230}
\definecolor{metacolor}{RGB}{238,246,238}
\definecolor{omicscolor}{RGB}{246,238,246}
\usepackage{hyperref}
\usepackage{url}
\usepackage{booktabs}
\usepackage{graphicx}
\usepackage{amsthm}
\usepackage{array}
\newtheorem{proposition}{Proposition}[section]
\newtheorem{corollary}[proposition]{Corollary}
\theoremstyle{remark}
\newtheorem{remark}[proposition]{Remark}
\usepackage{amssymb}
\usepackage{placeins}

\begin{document}

%

%

\twocolumn[

\aistatstitle{How High Is 0.6? Floors, Ceilings, and Headroom in Interpretability Probing}

\aistatsauthor{Pranjal Garg}

\aistatsaddress{Independent Researcher } ]

\begin{abstract}
Probes are the workhorse of interpretability. If a model's hidden states predict a variable, the model is said to represent it. But a probe score has no fixed meaning. An $R^2$ of 0.6 may only reflect what the input already gives away, and the same score can mean different things on different data. We propose reading every probe score against two reference points: a floor, what a declared set of simple inputs already predicts, and a ceiling, what the full input can predict. The gap between them, the headroom, is the range in which a probe can show that a model computes something beyond the simple inputs. We prove that headroom vanishes in two ways: the target stops depending on a hidden variable the model must infer, or the input stops revealing it. We test this on transformers trained for in-context meta-analysis, which must infer the hidden heterogeneity between studies to weight them correctly, and where both reference points are known. Under distribution shift, probe scores fall and prediction error rises $12$--$15\times$, yet the model recovers a similar share of the headroom, indicating that the data lost information, not the representation. We then analyze the real models.The single-cell foundation model scGPT encodes biological variability only partially. We also revisit four influential LLM probing studies, which claim that models represent geography, the state of an Othello board, truth, and the demographics of their users. Against a floor computed from the input text alone, some of these claims hold, while others are largely explained by the text itself.
\end{abstract}

\section{Introduction}
Understanding what information is represented inside a complex system is a central problem across biological and artificial intelligence. In neuroscience, decoding analyses use neural population activity to ask whether variables such as sensory features, cognitive states, or behavioral quantities are represented in the brain; analogous methods are now widely used to study the hidden representations of learned artificial networks \citep{quian2009extracting,naselaris2011encoding, belinkov2022probing}. In modern AI, where high-dimensional internal states are otherwise difficult to interpret directly, \emph{probes} provide a particularly simple and scalable readout: a classifier or regressor is trained to predict a quantity of interest from a model's hidden states, and its predictive performance is used to assess how strongly that quantity is represented \citep{alain2016understanding, hewitt2019designing}. Such measurements underlie claims about what models encode, how representations emerge during learning and how information is distributed across layers \citep{tenney2019bert, saphra2019understanding}. They therefore carry substantial scientific weight, as conclusions about the internal organization of both biological and artificial systems can depend on what a probe score is taken to mean.

\begin{figure*}[!htbp]
    \centering
    \includegraphics[width=0.65\textwidth]{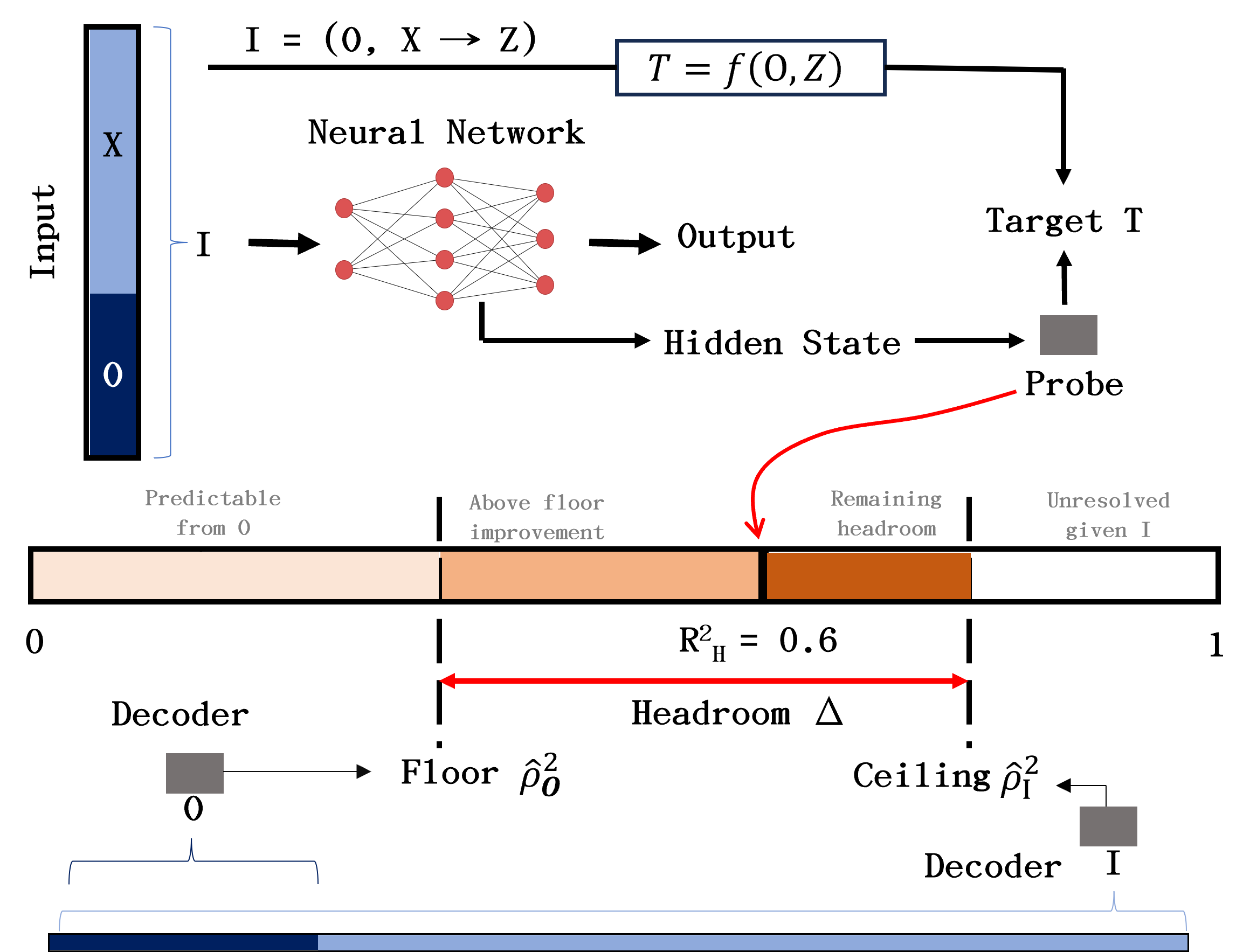}
    \caption{\textbf{Schematic of the floor--ceiling normalized probe $R^2$.} A probe's raw decodability ($R^2_H$) for a target $T$ is normalized between an estimated floor ($\hat\rho_O^2$), achievable from restricted information $O$, and an estimated ceiling ($\hat\rho_I^2$), achievable using full information $I$. The gap between them defines the  headroom ($\Delta$).}
    \label{fig:calibration_schematic}
\end{figure*}

Yet this widespread use exposes a basic interpretive problem: a raw probe score has no intrinsic scale of evidential strength. A high \(R^2\) may arise because the target is predictable from readily available input information rather than because the representation captures the computation of interest, and the same numerical score can carry different implications across distributions as this baseline predictability changes \citep{ravichander2021probing, elazar2021amnesic, hewitt2019designing}. Interpreting probe scores therefore requires knowing not only what is decodable but also how much decodability is achievable without the computation of interest and how much is achievable from all available information. We study this distinction in a controlled evidence-aggregation task. We train transformer models in a random-effects meta-analysis setting where the sufficient statistics and the oracle Bayes estimator conditional on the true heterogeneity parameter are known in closed form. Under a distribution shift, the raw decodability of the sufficient statistic falls and the prediction error rises, and the raw scores alone cannot say whether the representation degraded or the data became less informative.

To separate computation from statistical artifact, \textbf{we introduce a framework that normalizes probe scores against the bounds implied by the information sets that the scientific question specifies.} We partition the available information into restricted information ($O$) and additional evidence ($X$) used to infer a latent variable ($Z$). We evaluate probe decodability between two population reference points (under the true data distribution): a \emph{floor} ($\rho_O^2$), which measures the predictability achievable from the baseline alone, and a \emph{ceiling} ($\rho_I^2$), which measures the maximum predictability achievable given full information ($I$), which represents all available evidence. We define the gap between these reference points as \emph{headroom} ($\Delta$): the maximum variance a probe could explain by successfully inferring the latent variable $Z$ from $X$ (Figure~\ref{fig:calibration_schematic}). 

The theoretical result demonstrates that this headroom can vanish through two mechanisms: \emph{weak target sensitivity} (where the target ceases to depend on the latent variable) and \emph{weak identification} (where the evidence becomes too noisy to identify the latent variable). When headroom vanishes, raw probe scores can cease to provide evidence about the computation of interest. We test the estimation procedure in the meta-analysis task and a synthetic single-cell model, where the generative model is known and the reference points can be computed exactly by numerical integration. We then apply it to real single-cell data and to published LLM probes, where the reference points can only be estimated from samples. Analyses of a fine-tuned chemistry model and recorded neural activity appear in Appendix~\ref{app:contrast_sensitivity}.

\textbf{Contributions}: (1) we introduce a  floor, ceiling and  headroom normalization for interpretability probes and identify two distinct mechanisms by which this headroom vanishes: weak target sensitivity and weak identification of the latent variable; (2) we show that under distribution shift raw decodability can reflect reduced headroom rather than a degraded representation; (3) we test this estimation procedure in different domains, showing how the choice of floor and ceiling distinguishes different scientific questions; (4) we show that the same raw probe score can imply different amounts of recovered  headroom across settings.

\section{Related Work}
\label{sec:related}

\paragraph{Probing and its baselines.}
A probe trains a simple readout on hidden states and treats its predictive performance as evidence that a variable is represented \citep{alain2016understanding,belinkov2022probing}. Raw probe accuracy underdetermines this conclusion in three ways. Expressive probes can learn the task themselves \citep{hewitt2019designing}. Random encoders already support high accuracy \citep{zhang2018language,wieting2019no}. Decodable information need not be used by the model \citep{ravichander2021probing,elazar2021amnesic,kumar2022probing}. Proposed remedies either control probe capacity \citep{hewitt2019designing,voita2020information} or estimate information content \citep{pimentel2020information}. \citet{pimentel2020information} note that, by the data-processing inequality, no representation carries more information about a target than its input does.

Closest to our floor is conditional probing \citep{hewitt2021conditional}. Building on $\mathcal{V}$-information \citep{xu2020theory}, it measures what a representation adds beyond a baseline representation. We differ in three respects:
(i) the floor is defined on a declared restricted \emph{input} set $O$ chosen by the scientific question;
(ii) we pair it with a ceiling from the full input $I$, turning the data-processing bound into an estimable reference point against which scores are normalized;
(iii) we show that both reference points move with the data regime, so a fixed baseline can misattribute a regime change to the representation; we show this in Section~\ref{sec:meta} and Appendix~\ref{app:conditional}, where the conditional gain halves under a distribution shift while the headroom-normalized score does not.

\paragraph{Interpreting scientific foundation models.}
Probing and related readouts are increasingly being used to make scientific claims about learned models, for example, that protein language models encode structure \citep{rives2021biological} and that single-cell foundation models encode regulatory or cell-state information \citep{cui2024scgpt,theodoris2023transfer}. Several audits find that such models often fail to beat simple baselines \citep{kedzierska2025zero,boiarsky2024deeper,ahlmanneltze2025deep}. A recent audit of scGPT and Geneformer finds that the apparent biological structure in attention adds nothing beyond trivial gene-level statistics \citep{kendiukhov2026systematic}. These studies choose baselines case by case and report raw differences. We formalize the question they share: how much a representation recovers beyond a declared simple baseline, relative to what is predictable from the available inputs, and apply it to scGPT, a fine-tuned TabPFN \citep{hollmann2025accurate,grinsztajn2026tabpfn}, and neural population activity. The latter two appear in Appendix~\ref{app:contrast_sensitivity}.

\paragraph{Distribution shift and causal relevance.}
Interpretations that are faithful in distribution can fail under systematic shift \citep{friedman2023interpretability}. Causal methods such as nullspace projection and interchange interventions test whether decodable information is used \citep{ravfogel2020null,geiger2024finding}. Our framework is complementary. It does not test causal use; it specifies, before any intervention, how much evidential weight a correlational score can carry in a given regime. It also shows that a drop in raw decodability under distribution shift can reflect reduced headroom rather than a degraded representation.

\paragraph{Controlled testbeds with known sufficient statistics.}
Meta-trained sequence models approximate Bayes-optimal inference \citep{mikulik2020meta,muller2022transformers}, and transformers learn statistical estimators in context \citep{garg2022can,akyurek2022learning,bai2023transformers}. Probes have been used to compare their internal quantities with classical estimators \citep{teh2025solving}. We build on this line with random-effects meta-analysis \citep{higgins2002quantifying}.

\section{Controlled meta-analysis setup}
\label{sec:setup}

\subsection{Task and model}
\paragraph{Controlled random-effects meta-analysis.}
Each example contains $K$ studies estimating an underlying effect $\mu$.
Study-specific effects vary around $\mu$ with between-study heterogeneity
$\tau^2$, while each reported estimate $x_i$ is additionally corrupted by
study-specific sampling variance $\sigma_i^2$. Marginally,
$
x_i\mid \mu,\tau^2,\sigma_i^2
\sim
\mathcal N(\mu,\sigma_i^2+\tau^2)$. If $\tau^2$ were known, each study would receive precision weight $w_i=\frac{1}{\sigma_i^2+\tau^2}$, and all information relevant to estimating $\mu$ would be summarized by the sufficient statistics $S_1=\sum_i w_i$ and $S_2=\sum_i w_i x_i$. Here, $S_1$ is the total precision of the studies and $S_2$ is their
precision-weighted effect. We therefore use $\log S_1$ as the primary target for representation probing, and separately probe $\tau^2$ to test whether the model represents the hidden heterogeneity needed to construct the weights.

We define \textbf{distribution shift} as an out-of-distribution (OOD) change from a heterogeneity-dominated regime to a sampling-variance-dominated regime. We quantify the regime by the sampling-variance share $r_i=\frac{\sigma_i^2}{\sigma_i^2+\tau^2}$, whose median increases from 0.31 in distribution (ID) to 0.71 OOD. Thus, the ID regime is predominantly heterogeneity-driven, whereas the OOD regime is predominantly sampling-variance-driven. The distribution shift changes which source of uncertainty determines the precision weights. For full details, see Appendix~\ref{app:bayes}.

\paragraph{Model and training.} We train causal pre-norm transformers with attention and MLP blocks at five sizes ranging from $3.3$k to $794$k parameters (xs, s, m, l, xl). The transformer model receives $(x_i,\log\sigma_i^2)$ as input, and the true heterogeneity $\tau^2$ is hidden. It predicts the overall effect $\mu$ under a causal mask (Appendix~\ref{app:training}).

\paragraph{Probing and architectural controls.} 
All probe scores and estimated reference points are computed using a bank of linear and nonlinear regressors (ridge, MLP, and gradient boosting) for every layer (Appendix~\ref{sec:probe-bank}). We report the standard $R^2$ for every held-out target. We probe $\log S_1$ rather than $S_1$ because the logarithm reduces target skew from $3.9$ to $-0.6$.

\subsection{Raw probe behavior under the distribution shift}
\begin{figure}[!htbp]
    \centering
    \includegraphics[width=\linewidth]{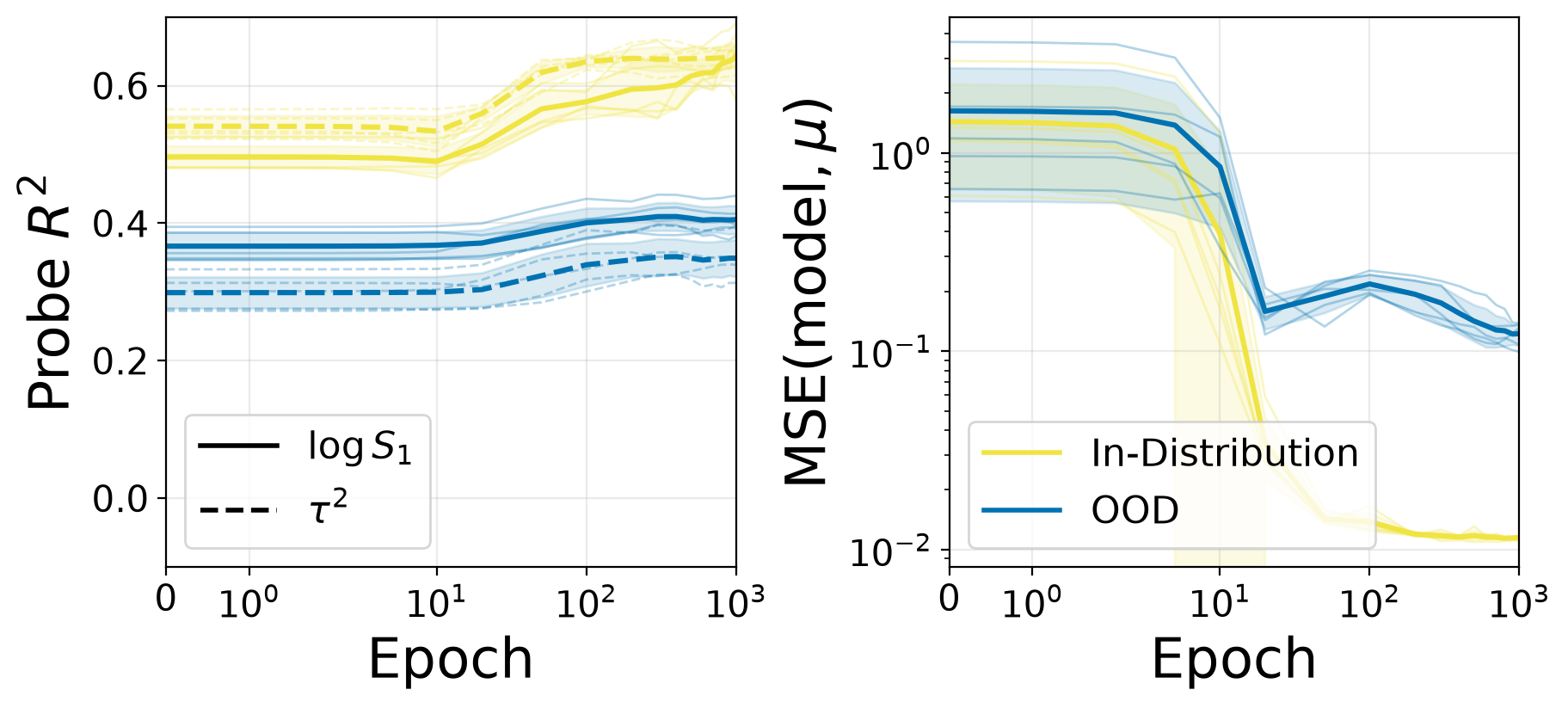}
    \caption{\textbf{In-distribution and OOD learning trajectories.} Probe decodability for $\log S_1$ and $\tau^2$ and final-position model MSE across training (xl model).}
    \label{fig:dissociation}
\end{figure}

Table~\ref{tab:dissociation} and Figure~\ref{fig:dissociation} show that the
distribution shift degrades all three measured quantities. For models s
through xl, probe $R^2$ for $\log S_1$ falls (from $0.64$ to $0.40$ for xl),
probe $R^2$ for $\tau^2$ falls (from $0.64$ to $0.34$ for xl), and prediction
MSE at the final position rises $12$--$15\times$. In distribution, model error is within $3\%$ of
the oracle conditional-Bayes risk; out of distribution it is
$2.1$--$2.6\times$ that risk. Both targets nonetheless remain partly decodable
OOD ($R^2\approx0.29$--$0.40$ for $\log S_1$ and $0.26$--$0.34$ for $\tau^2$).
The xs model is excluded from these comparisons because it does not reliably
learn either target in distribution ($R^2 = 0.28 \pm 0.20$ for $\log S_1$).

The raw scores admit two explanations: the representation degraded, or the
OOD data carry less information about the targets. A further concern applies
to $\log S_1$. When sampling variance dominates,
$S_1=\sum_i(\sigma_i^2+\tau^2)^{-1}\approx\sum_i\sigma_i^{-2}$ is predictable
from the sampling variances alone, so the OOD score of $0.40$ might require no
inference about $\tau^2$. If so, the best prediction of $\log S_1$ from the
sampling variances alone should rise sharply under the shift. We formalize
this quantity as the floor $\rho^2_O$ in Section~\ref{sec:theory};
Section~\ref{sec:collapse} shows that it rises only from $0.01$ to $0.09$.
Raw probe scores cannot separate these explanations;
Section~\ref{sec:theory} introduces a normalization that does.

\section{Normalizing probe decodability}

\subsection{Floor, ceiling and headroom}
\label{sec:theory}

Let $T\in L^2$ be a non-constant target quantity whose representation we wish to audit. Let $O$ denote restricted information and $I$ the full information available to the system, with $\sigma(O)\subseteq\sigma(I)$. Equivalently, we may write $I=(O,X)$, where $X$ denotes the additional evidence beyond $O$. The representation under study is denoted by $H$; for deterministic models, $H=\Phi(I)$.

We define two population reference points,
\[
\rho_O^2:=1-\frac{\mathbb E[\operatorname{Var}(T\mid O)]}{\operatorname{Var}(T)}, \qquad \rho_I^2:=1-\frac{\mathbb E[\operatorname{Var}(T\mid I)]}{\operatorname{Var}(T)},
\]
and their difference $\Delta:=\rho_I^2-\rho_O^2\geq0$ (Propositions~\ref{prop:ceiling} and~\ref{prop:floor}).

Here $\rho_O^2$ is the maximum population $R^2$ attainable from the restricted information $O$, whereas $\rho_I^2$ is the maximum attainable from the full information $I$. We call $\rho_O^2$ the \emph{floor}: it is not a lower bound on an arbitrary probe score, but the level of predictability already available without using the additional evidence $X$. Consequently, a probe score at or below $\rho_O^2$ does not by itself establish that the representation contains predictive information about $T$ beyond $O$ (Proposition~\ref{prop:floor}; Remark~\ref{rem:evidential_headroom}).

The headroom $\Delta$ quantifies how much additional predictability $I$ supplies beyond $O$. It equals the normalized variance of the Bayes update that $X$ supplies about $T$ beyond $O$ (Proposition~\ref{prop:identity}), and it vanishes through weak target sensitivity or weak identification (Corollary~\ref{cor:channels}; Section~\ref{sec:collapse}).

\begin{figure}
    \centering
    \includegraphics[width=\linewidth]{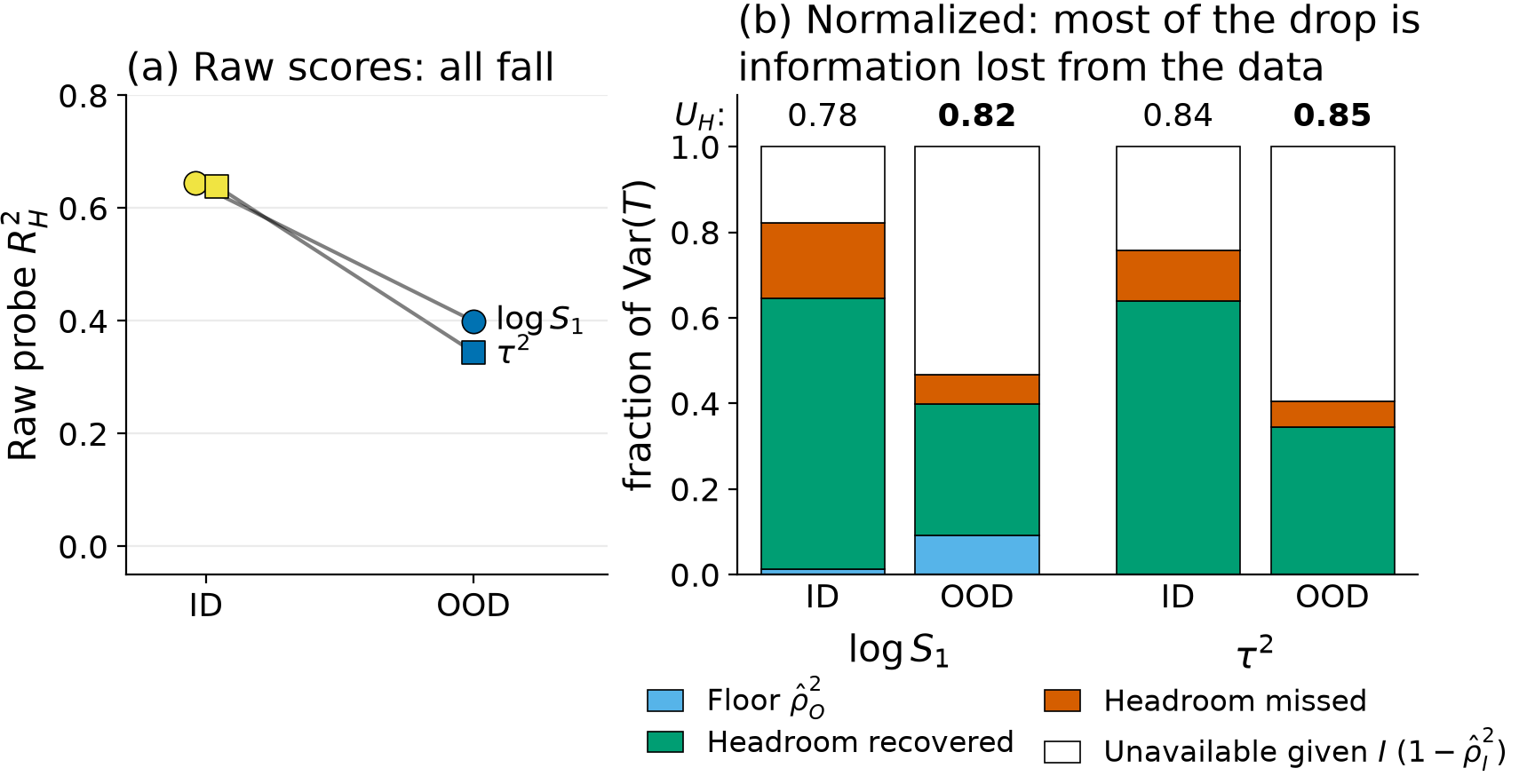}
    \caption{\textbf{Raw and normalized probe performance under distribution shift.} Raw probe scores decrease OOD, while normalization by the floor and ceiling shows that most of the decline reflects reduced available information.}
    \label{fig:main_meta}
\end{figure}

For a representation $H$, let $R_H^2$ denote the held-out probe performance for predicting $T$ from $H$. When $H$ is measurable with respect to $I$, no predictor based on $H$ can exceed the population 
ceiling $\rho_I^2$ (Proposition~\ref{prop:ceiling}). When $\Delta>0$, we therefore normalize representation decodability as $U_H:=\frac{R_H^2-\rho_O^2}{\rho_I^2-\rho_O^2}$. Here $U_H$ measures probe performance relative to the  headroom available beyond $O$. Each reference point has three versions. The \emph{population} reference points are defined above. Second is \emph{estimating} the reference points $\hat\rho^2_O,\hat\rho^2_I$, by fitting a fixed decoder bank to $O$ and to $I$ on held-out data (the \emph{estimation procedure}; Appendix~\ref{sec:probe-bank}). When the generative model is known (meta-analysis, synthetic single-cell), we additionally compute them \emph{exactly} by numerical integration. In both of them, estimates closely track exact values (Appendix~\ref{sec:exact-endpoint-certificate}, Figure~\ref{fig:exact-vs-empirical}).  For meta-analysis and single-cell test cases, we estimate the reference points with feature representations \(\phi_O(O)\) and \(\phi_I(I)\). The feature maps are fixed deterministic functions of their corresponding information sets (Appendix~\ref{app:endpoint-features}). Estimated ceilings are conservative, so an underestimated ceiling inflates $U_H$.

\begin{figure}[htbp]
    \centering
    \includegraphics[width=\linewidth]{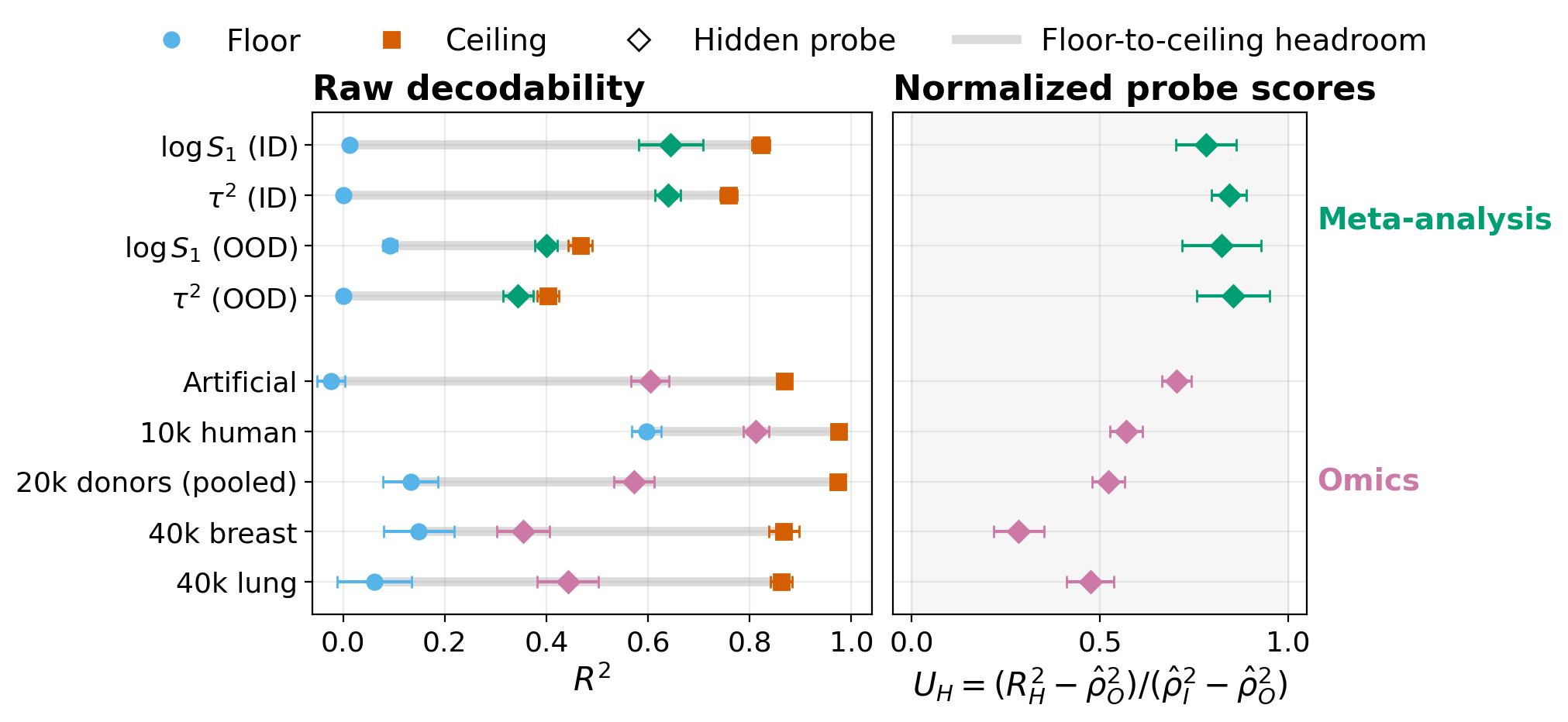}
    \caption{
    \textbf{Raw and headroom-normalized decodability across domains.}
    Left: $R^2$ of the floor, the ceiling and the representation probe. Right: the normalized decodability $U_H=(R_H^2-\hat\rho_O^2)/(\hat\rho_I^2-\hat\rho_O^2)$, the fraction of the headroom recovered by the probe. Meta-analysis reference points are exact (Table~\ref{tab:exact-endpoints}). ``Artificial'' is the synthetic single-cell model and ``Omics'' the single-cell data of Table~\ref{tab:cross-domain-calibration-legacy}.
    }
    \label{fig:calibration}
\end{figure}

A natural alternative is conditional probing \citep{hewitt2021conditional} with baseline $O$ instead of the model's own non-contextual embedding layer: the gain $R^2_{O,H}-R^2_O$ from adding $H$ to a probe on $O$. Measuring this gain from the floor instead (in the population, $R^2_O=\rho^2_O$) and normalizing it by the headroom gives $U_{H|O} := (R^2_{O,H}-\rho^2_O)/(\rho^2_I-\rho^2_O)$. In the population $U_H \le U_{H|O} \le 1$, with equality on the left when $H$ retains everything in $O$ relevant to $T$ (Appendix~\ref{app:conditional}).

When an estimated floor is well below chance for classification or well below zero for regression, the headroom normalization is not applicable, and we report the raw probe performance instead. Floors only slightly below chance or zero, as expected from estimation noise when the population floor is near zero (Figure~\ref{fig:exact-vs-empirical}), are kept unclipped (Appendix~\ref{sec:probe-bank}).

\paragraph{Evaluation settings.} We specify the contrast $(O,I,T)$ in settings with progressively weaker access to ground truth: the controlled meta-analysis and a synthetic single-cell model, where the reference points are exact (Appendix~\ref{sec:exact-endpoint-certificate}), then real single-cell representations and published LLM probes. The contrast is chosen from the scientific question, independently of the probe scores, and normalized results are interpreted relative to it. Where no unique contrast exists, we evaluate nested baselines $O_1\subset O_2\subset\dots\subset O_k$. Appendix~\ref{app:contrast_sensitivity} applies this to a fine-tuned chemistry model and recorded neural activity.

\subsection{Normalized probes across domains}
\subsubsection{Meta-analysis}
\label{sec:meta}
Using the controlled random-effects model defined above, we take the
sufficient statistic $\log S_1$ as the probe target,
$T=\log \sum_{i=1}^{K}\frac{1}{\sigma_i^2+\tau^2}$. The scientific question is
whether the transformer's hidden representation contains information about
the precision-weighting computation beyond what the reported sampling
variances alone predict. We therefore define the restricted information as
$O=\sigma_{1:K}^2$ and the full information as $I=(\sigma_{1:K}^2,x_{1:K})$,
where the observed effects $x_i$ provide the additional evidence from which
the latent heterogeneity $\tau^2$ can be inferred. We also probe $\tau^2$
directly. Because $\tau^2$ is independent of $O$, its floor is zero, and its
ceiling measures how much heterogeneity the reported effects reveal. The
distribution shift changes the relative contribution of sampling variance and
heterogeneity, so it tests whether raw decodability tracks the information
available in each regime.

Figure~\ref{fig:calibration} and Table~\ref{tab:cross-domain-calibration-legacy}
place each raw probe score between a floor and a ceiling set by a
domain-specific contrast. In the \textbf{meta-analysis}, the contrast asks
whether the hidden state carries heterogeneity inferred from the reported
effects beyond what the sampling variances predict. In distribution, the
probe recovers $78\%$ of the headroom for $\log S_1$ and $84\%$ for $\tau^2$.
Under the shift, raw decodability falls for both targets
($0.64\to0.40$ for $\log S_1$, $0.64\to0.34$ for $\tau^2$;
Table~\ref{tab:dissociation}), and so do the exact ceilings
($0.82\to0.47$ and $0.76\to0.40$; Table~\ref{tab:exact-endpoints}). The
fraction of headroom recovered does not decrease ($U_H = 0.78\to0.82$ for
$\log S_1$, $0.84\to0.85$ for $\tau^2$; Figure~\ref{fig:main_meta}).
The lower raw scores therefore reflect less recoverable information in the
data, not a less informative representation (Section~\ref{sec:collapse}).

Conditional probing alone would suggest the opposite. With a linear decoder throughout, the conditional gain roughly halves under the shift ($0.58 \to 0.31$ for $\log S_1$; $0.64 \to 0.34$ for $\tau^2$), which alone suggests a degraded representation. The headroom (estimated with linear decoder) also halves ($0.62 \to 0.33; 0.66 \to 0.36$), and once the gain is divided by it, $U_{H|O}$ stays flat ($0.95 \to 0.95$; $0.97 \to 0.95$), corresponding to $U_H$ (Appendix~\ref{app:conditional}).

\subsubsection{Genomics and single-cell data}
\label{sec:genomics}
Single-cell RNA sequencing has the same structure as meta-analysis. Each example is a gene $g$ measured in $K$ cells. The normalized count of a cell varies for two reasons: technical noise, which decreases as the depth of the cell's sequencing $d_{gi}$ increases, and the biological variation from cell-to-cell, summarized by the overdispersion of the gene $b_g$. Under a Gamma--Poisson model the normalized count has variance $\lambda_g(d_{gi}^{-1}+b_g)$, where $\lambda_g$ is the gene's mean expression. Technical noise $d_{gi}^{-1}$ therefore plays the role of the sampling variance $\sigma_i^2$, and overdispersion $b_g$ the role of heterogeneity $\tau^2$. The target is the analogue of $\log S_1$, the gene's effective precision $T_g=\log\sum_{i=1}^{K}(d_{gi}^{-1}+b_g)^{-1}$, which depends on the hidden overdispersion. The restricted information is the sequencing depths alone, $O_g=d_{g,1:K}$, and the full information adds the observed counts, $I_g=(y_{g,1:K},d_{g,1:K})$, from which overdispersion can be inferred. We ask whether frozen scGPT \citep{cui2024scgpt} encodes $T_g$ beyond what depth alone predicts.

We first ask this on synthetic genes ($K=40$; Appendix~\ref{app:sc_gen_model}) and then on four real scRNA-seq datasets \citep{10xgenomics2021pbmc, 10xgenomics2024pbmcdonors, 10xgenomics2022tumors}. In real data, the true $b_g$ is unknown, so we split the cells of each dataset into half at random. Overdispersion is estimated from the first half and defines the target, while the depths, counts and scGPT inputs come from the second half, so the ceiling cannot simply recompute its own target. $K$ is then the number of cells in the second half.

For \textbf{single-cell data}, raw scores are misleading because the floor differs across datasets. In the 10k human dataset, depth alone already reaches $R^2=0.60$, so scGPT's raw score of $0.81$ recovers $57\%$ of the headroom. In the 20k donors dataset the raw score is much lower ($0.57$), but the floor is only $0.13$, so scGPT recovers a similar $52\%$. Across the four datasets, raw scores range from $0.35$ to $0.81$, while the fraction of headroom recovered ranges from $0.28$ to $0.57$; on synthetic genes it is $0.70$. scGPT therefore encodes biological overdispersion beyond sequencing depth, but only partially. scGPT probes use a 64-component PCA fitted on training rows only; all floors, ceilings and other probes use the full feature vectors (Appendix~\ref{app:endpoint-features}).

\subsubsection{Published LLM probes}
\label{sec:llm_probes}
We compute floors for four published LLM probes, using descriptions of the input. $T$ is the published target, $I$ is the input the model
received, and $O$ is a surface description of it: $n$-grams of a
name, statement, or conversation, or move summaries that never simulate Othello flips.
The floor is what the decoder bank learns from $O$ on the probe's own training split,
so it needs only the released inputs, labels, and splits. For Othello, the ceiling is
exactly 1, since the board follows from the moves; elsewhere we set it to 1, which can
only lower $U_H$, while a richer floor readout could lower it further. We therefore
read $U_H$ here as an approximate value and compare the raw scores with the floor directly where
the headroom is small. For categorical targets, we use accuracy in place of $R^2$.
Figure~\ref{fig:llm_highlights} summarizes the results. See full results in Appendix~\ref{app:llm}.

\begin{figure*}[t]
\centering
\includegraphics[width=0.8\textwidth]{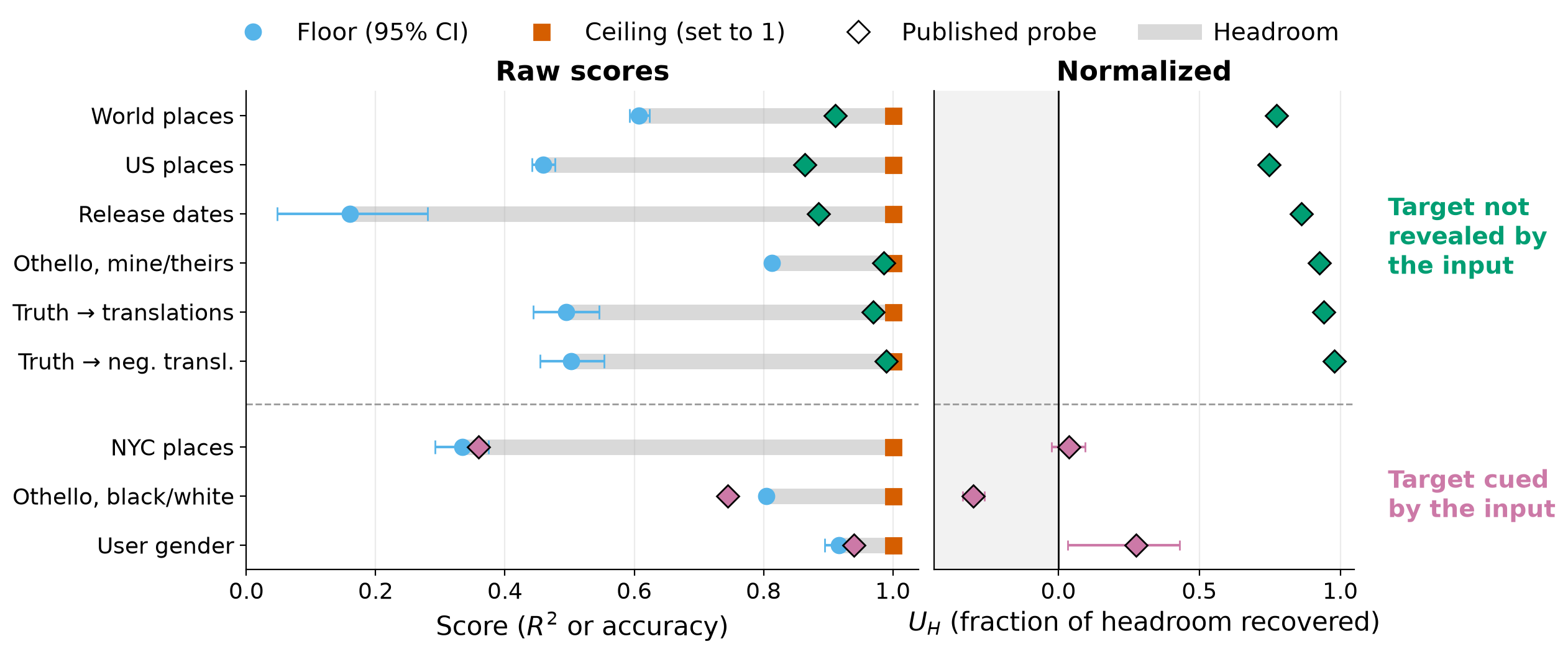}
\caption{\textbf{Probe normalization puts published LLM probe scores in perspective.}
Left: floor computed from a description of the input ($n$-grams;
disc placement without flips for Othello), with 95\% CI, the probe score, and
the ceiling set to 1 ($R^2$ for places and dates, accuracy otherwise). Right: normalized probe score, $U_H$, with error bars from the floor's CI.} 
\label{fig:llm_highlights}
\end{figure*}

\textbf{Space.} \citet{gurnee2024language} report that Llama-2 represents the coordinates of named places. A floor computed from a description of the place name, its character 2--4-gram counts, reaches $R^2 = 0.61$ for world places and 0.46 for US places, yet the published probes recover 63–77\% of the headroom above it. For New York City this floor is 0.33 (95\% CI 0.29–0.37), and the published scores (0.22 for 7B, 0.36 for 70B) lie within or below it, so they cannot separate a city-level map from borough names and numbers in the strings.

\textbf{Othello.} \citet{li2022emergent} report that Othello-GPT represents the board.
A floor that knows which player placed each disc, but never simulates flips, reaches
0.80 accuracy (chance 0.46). The linear black/white/empty probe (0.74) falls below it
($U_H=-0.30$), so its accuracy is accounted for by disc placement. A nonlinear probe
recovers 45\% of the headroom with our decoders. In the mine/theirs/empty
encoding \citep{nanda2023emergent}, a linear probe reaches 0.99 and recovers 92\% of the
headroom above the floor from the same features (0.81 for this target), indicating that the effect of flips is represented linearly in the
relative encoding.

\textbf{Truth.} \citet{marks2023geometry} report that LLMs represent whether a statement is true along a linear direction, and that a probe trained on one dataset of statements still works on another. We compute the floor from a description of each statement, its character and word $n$-grams. Within a single dataset, this floor is near chance for most datasets but reaches 0.97--0.99 on \texttt{larger\_than} and \texttt{smaller\_than} datasets, whose truth labels follow from the wording of the numbers. Probes trained and tested on those two datasets therefore show little beyond the text. Transfer between datasets is more informative, because wording learned on one dataset does not carry over to another. A floor trained on the two number datasets and tested on Spanish--English translations is at chance (0.49), while the published probes reach 0.97 ($U_H=0.94$). Trained on \texttt{cities} and \texttt{neg\_cities} and tested on negated translations, the floor is again at chance (0.50), while the probes reach 0.96--0.99 ($U_H$ up to 0.98). 

\textbf{User attributes.} \citet{chen2024designing} report probes that read a user's
age, gender, education and socioeconomic status from Llama-2-Chat activations. The
floor, an $n$-gram classifier on the conversation, reaches 0.92--0.98, and every published reading
probe (0.94--0.98) lies within its 95\% interval; every control probe falls below its point estimate.
The floor is as high from the user's turns alone, and falls to 0.76--0.90 when trained
on GPT-3.5 conversations and tested on Llama-2-Chat ones. However, the role-played conversations were written to reveal the attribute, so these accuracies can neither establish nor rule out that the model internally represents the user's attributes.

\textbf{Summary.} Probe normalization puts these published scores in perspective (Figure~\ref{fig:llm_highlights}). Probe scores reach well above the floor when the target is not revealed by the surface form: coordinates at world scale, the effect of flips on the board, and truth transferred across datasets. When the target is written or cued in the input, such as borough names and numbers for New York City places, disc placement for the absolute board encoding, or attributes the role-played user writes into the conversation, the floor alone accounts for much of the score.

\subsection{Mechanisms for vanishing headroom}
\label{sec:collapse}
The headroom $\Delta$ is the predictability that the full information $I$
adds beyond the restricted baseline $O$. It can vanish in two distinct ways
(Corollary~\ref{cor:channels}). Let $T = f(O,Z)$ for a latent variable $Z$,
and let $I=(O,X)$, where $X$ is the additional evidence about $Z$. \textbf{ (1) Weak target sensitivity:} $T$ barely depends on $Z$. If
    $f$ is Lipschitz in $Z$ with constant $L(O)$, the headroom is bounded by
    $\mathbb E[L(O)^2\,\mathrm{Var}(Z\mid O)]/\mathrm{Var}(T)$, so knowing $Z$ adds little even when $X$
    reveals it perfectly. The floor rises toward the ceiling. \textbf{ (2) Weak identification:} $X$ reveals little about $Z$ beyond
    $O$. If $f$ varies by at most $M$ over $Z$, the headroom is bounded by
    $M^2\,\mathrm{MI}(X;Z\mid O)/(2\,\mathrm{Var}(T))$, where $\mathrm{MI}$ is the conditional mutual information in nats,
    so $T$ may depend strongly on $Z$, but the data cannot pin $Z$ down. The
    ceiling falls toward the floor.

\paragraph{Which mechanism drives the meta-analysis shift?}
The distribution shift (Section~\ref{sec:setup}) pushes in both directions.
When sampling variance dominates, $\log S_1=\log\sum_i(\sigma_i^2+\tau^2)^{-1}$
depends less on $\tau^2$ (weaker target sensitivity), and the per-study Fisher
information about $\tau^2$, which scales as $(\sigma_i^2+\tau^2)^{-2}$, drops,
so the reported effects reveal less about $\tau^2$ (weaker identification;
Remark~\ref{rem:fisher}). The exact reference points show that weak
identification dominates (Table~\ref{tab:exact-endpoints}). The headroom for
$\log S_1$ shrinks from $0.81$ to $0.37$. The floor rises only from $0.01$ to
$0.09$, while the ceiling falls from $0.82$ to $0.47$, so most of the loss
comes from the ceiling. For $\tau^2$ itself the target is the latent
variable, so target sensitivity cannot change; its floor stays at zero and its
ceiling falls from $0.76$ to $0.40$, a pure identification loss (Figures~\ref{fig:main_meta} and~\ref{fig:calibration}).
\begin{figure}[!htbp]
    \centering
    \includegraphics[width=\linewidth]{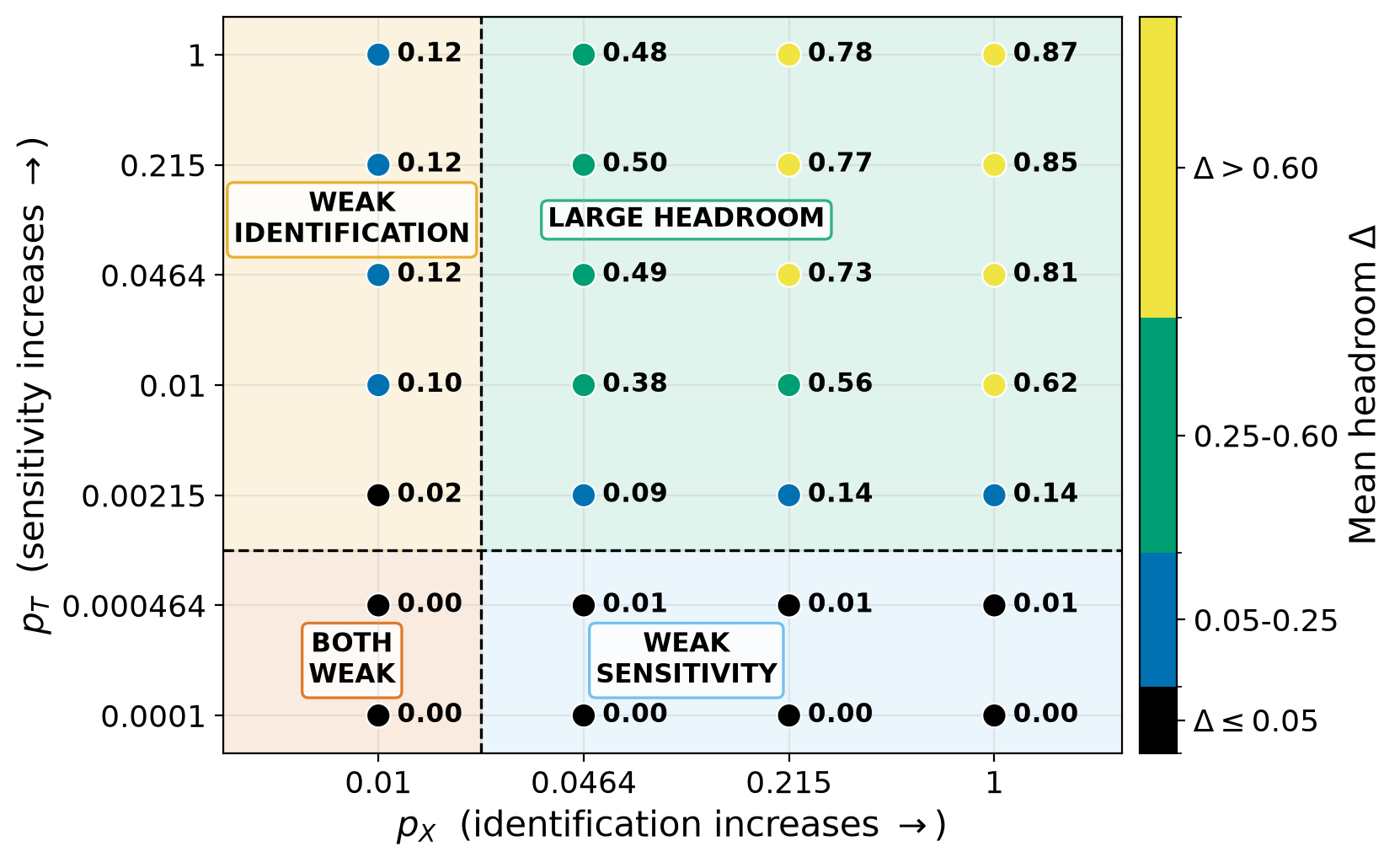}
    \caption{\textbf{Headroom as evidence quality and target sensitivity
    vary, in the synthetic single-cell model (with binomial thinning).} Evidence quality $p_X$ controls how well the data
    identify the latent variable; target sensitivity $p_T$ controls how
    strongly the target depends on it. Low $p_X$ gives weak identification,
    low $p_T$ gives weak target sensitivity, and low values of both give
    both.}
    \label{fig:collapse}
\end{figure}
\paragraph{Separating the mechanisms in a controlled grid.}
In the meta-analysis both mechanisms move together. The synthetic single-cell
model (Section~\ref{sec:genomics}; Appendix~\ref{app:sc_gen_model}) lets us vary them independently. The latent variable is the gene's biological overdispersion, $Z=b_g$; the restricted information is sequencing depth,
$O_g=d_{g,1:K}$; the additional evidence is the observed counts, $X_g=y_{g,1:K}$; and the target is $T_g$.
Reducing sequencing depth by binomial thinning, retaining a fraction $p$ of
counts, lowers the effective depth from $d_{gi}$ to $p\,d_{gi}$. In real data this
affects both the evidence and the target at once. In simulation we apply it
separately (Appendix~\ref{app:sepparating}): thinning the counts used as evidence ($p_X$) degrades inference of
$b_g$, giving weak identification, whereas thinning the
depths entering the target ($p_T$) lets technical variance dominate
overdispersion, giving weak target sensitivity; the restricted information is never thinned. Figure~\ref{fig:collapse}
shows the resulting headroom over the $(p_X, p_T)$ grid, with estimated
reference points. Headroom vanishes when either parameter is small and stays
large only when both are large, so a loss of headroom can be traced to its
cause.

\section{Conclusion}
\label{sec:conclusion}
A probe score of $0.6$ has no fixed meaning. What it says about a representation depends on two quantities: how much of that predictability a declared baseline already provides, and how much the available data support in total. We made both explicit. \textbf{The floor $\rho^2_O$ is what a declared baseline $O$ already predicts, the ceiling $\rho^2_I$ is what the full input $I$ supports, and the headroom $\Delta = \rho^2_I - \rho^2_O$ between them is the range in which a probe score can show information beyond $O$.} Headroom vanishes in two ways: weak target sensitivity, when the target stops depending on the latent variable, and weak identification, when the evidence stops revealing it.

In the controlled meta-analysis, where the reference points are exact, normalization reverses the reading of a distribution shift. Raw decodability of $\log S_1$ and $\tau^2$ falls and prediction error rises, which alone suggests a degraded representation. The reference points show instead that the data lost information: the ceiling falls from $0.82$ to $0.47$ for $\log S_1$ and from $0.76$ to $0.40$ for $\tau^2$, while the floor for $\log S_1$ rises only from $0.01$ to $0.09$, so weak identification dominates. The representation recovers a similar share of the smaller headroom ($U_H = 0.78 \to 0.82$ for $\log S_1$; $0.84 \to 0.85$ for $\tau^2$).

The estimation recovers the exact reference points in two fully specified generative models and extends to settings where only samples are available. In frozen scGPT, similar raw scores correspond to different fractions of recovered headroom across datasets ($0.28$--$0.57$), so biological overdispersion is encoded only partially. For published LLM probes, a floor computed from the input text alone separates claims the text cannot explain (world-scale coordinates, the effect of flips in the relative Othello encoding, truth transferred across datasets) from claims it largely explains (New York City locations, the absolute Othello board, user attributes in role-played conversations).

The framework complements causal interventions: it bounds how much a correlational score can show before any intervention is run. \textbf{We recommend that interpretability studies report the floor, the ceiling and the fraction of headroom recovered alongside raw scores, state the contrast $(O,I,T)$ before looking at probe scores, and, where no unique baseline exists, report a nested sequence $O_1 \subset \dots \subset O_k$ so that readers can see how the choice of baseline affects the conclusions.}

\paragraph{Limitations.}
Normalized scores are defined relative to a stated contrast $(O,I,T)$. Different contrasts answer different scientific questions, so we report nested baselines (Appendix~\ref{app:contrast_sensitivity}) rather than a single choice. On observational data the reference points are estimated, and an underestimated ceiling or floor inflates $U_H$. For published LLM probes the ceiling cannot be estimated and is set to $1$, so $U_H$ there is approximate. Our controlled tasks are kept simple so that the reference points can be computed exactly, and the transformers trained on them are small.

\subsection*{AI use statement}

In this work, we used generative AI tools to aid and polish writing, for retrieval and discovery of related work, for research execution, and to draft sections of the paper. As part of research execution, we used coding agents to assist with implementing and debugging analysis code and with producing scientific figures. We did not use generative AI tools for generating synthetic datasets or for proving mathematical claims. We have reviewed all AI-assisted work. AI-assisted code and analyses were inspected and their outputs were checked against the corresponding experiments and results; AI-assisted text and figures were reviewed for consistency with the underlying methods, analyses, and findings. We take responsibility for the final content of this work, including text, claims, code, figures, and other artifacts produced with the aid of generative AI.



\bibliography{aistats}
\bibliographystyle{apalike}

\clearpage
\appendix
\thispagestyle{empty}

\onecolumn
\aistatstitle{Supplementary Materials}

\section{Controlled random-effects meta-analysis}
\label{app:bayes}

\paragraph{Generative model.}
Each example is a meta-analysis containing $K$ studies that attempt to
estimate the same underlying effect. We denote this effect by
\[
\mu\sim\mathcal N(0,1).
\]
The transformer models are trained to predict $\mu$. The individual studies need not have exactly the same true effect. Their
study-specific effects are
\[
\theta_i\sim\mathcal N(\mu,\tau^2),
\]
where
\[
\tau^2\sim\mathrm U(0,0.5)
\]
is the between-study heterogeneity. Large $\tau^2$ means that the true
effects $\theta_i$ vary substantially from study to study even before sampling
noise is introduced; small $\tau^2$ means that the studies are close to sharing a common true effect. Each study also has finite-sample measurement uncertainty. We draw
\[
n_i\sim\mathrm U\{5,\ldots,99\},
\qquad
s_i^2\sim\mathrm U(0.1,10),
\]
and define its sampling variance as
\[
\sigma_i^2=\frac{s_i^2}{n_i}.
\]
Thus studies with larger samples or smaller within-study variance are measured
more precisely. Finally, the reported study estimate is
\[
x_i\sim\mathcal N(\theta_i,\sigma_i^2).
\]
The difference between $x_i$ and the global effect $\mu$ therefore arises from
two distinct sources:
\[
\underbrace{\theta_i-\mu}_{\text{between-study heterogeneity}}
\qquad\text{and}\qquad
\underbrace{x_i-\theta_i}_{\text{sampling noise}}.
\]

Marginalizing over the unobserved study-specific effect $\theta_i$ gives
\[
x_i\mid\mu,\tau^2,\sigma_i^2
\sim
\mathcal N\!\left(
\mu,\,
\sigma_i^2+\tau^2
\right).
\]
The total uncertainty of study $i$ is therefore the sum of its sampling variance
and the common between-study heterogeneity. The transformer receives $(x_i,\log\sigma_i^2)$ for each study as an input. It does \emph{not} observe $\theta_i$, $\tau^2$, or $\mu$.
We hypothesize that it infers whatever information about $\tau^2$ is useful for
combining the reported studies.

\paragraph{Precision weighting.}
A study with lower total variance $\sigma_i^2+\tau^2$
should contribute more strongly to an estimate of $\mu$. Its corresponding
precision is the inverse of that variance,
\[
w_i
=
\frac{1}{\sigma_i^2+\tau^2}.
\]
Thus a large $w_i$ denotes a comparatively informative study, while a small
$w_i$ denotes a noisy or unreliable one. If $\tau^2$ were known, the $K$ studies could be summarized by two quantities:
\[
S_1=\sum_{i=1}^K w_i,
\qquad
S_2=\sum_{i=1}^K w_i x_i.
\]

The first, $S_1$, is the \emph{total precision}: it measures how much effective information the
collection of studies provides about $\mu$. The second, $S_2$,
is the precision-weighted sum of the observed effects. For fixed $\tau^2$, the likelihood for $\mu$ can be written as
\[
p(x_{1:K}\mid\mu,\tau^2,\sigma^2_{1:K})
\propto
\exp\!\left[
-\frac12
\left(
S_1\mu^2-2S_2\mu
\right)
\right].
\]
Consequently, once $\tau^2$ is known, the individual study observations contain no additional information about $\mu$ beyond $(S_1,S_2)$. Therefore, $(S_1,S_2)$ are sufficient statistics for estimating $\mu$ conditional on
$\tau^2$. This is why $\log S_1$ is useful as a representation-level audit target:
constructing it requires the model to combine the observed sampling variances
with information about the otherwise hidden heterogeneity $\tau^2$.

\paragraph{Inverse-variance estimate.}
The standard inverse-variance-weighted estimator is
\[
\hat\mu_{\mathrm{IVW}}
=
\frac{S_2}{S_1}.
\]
This denotes a weighted average of the reported study effects, where more precise studies receive larger weights. For example, if two studies have identical reported effects but one has much
smaller total variance, the latter contributes more strongly to
$\hat\mu_{\mathrm{IVW}}$. The heterogeneity $\tau^2$ plays an important role: increasing heterogeneity makes the study weights more similar because the common $\tau^2$ term begins to dominate their individual sampling variances.

\paragraph{Bayesian estimate.}
The generative model places the prior $\mu\sim\mathcal N(0,1)$,
whose precision is $1$. Combining this prior with the conditional likelihood gives
\[
\mu\mid x_{1:K},\sigma^2_{1:K},\tau^2
\sim
\mathcal N\!\left(
\frac{S_2}{S_1+1},
\frac{1}{S_1+1}
\right).
\]
The corresponding posterior mean is
\[
\hat\mu_{\mathrm B}(\tau^2)
=
\frac{S_2}{S_1+1}.
\]

As the prior on $\mu$ is centered at zero, the Bayesian estimate is the inverse-variance estimate multiplied by $S_1/(S_1+1)<1$. When the accumulated study precision $S_1$ is small, the prior has appreciable
influence and the estimate is pulled toward zero. As $S_1$ becomes large, $\frac{S_1}{S_1+1}\rightarrow1$, and the Bayesian and classical estimators become nearly identical.

Their corresponding conditional mean-squared errors are
\[
\operatorname{MSE}_{\mathrm B}
=
\frac{1}{S_1+1},
\qquad
\operatorname{MSE}_{\mathrm{IVW}}
=
\frac{1}{S_1}.
\]
The difference is therefore largest when little total precision has accumulated,
for example at short sequence prefixes, and becomes negligible when many precise
studies are available.

The expression $\frac{S_2}{S_1+1}$ requires the true $\tau^2$, because both $S_1$ and $S_2$ use $w_i=(\sigma_i^2+\tau^2)^{-1}$.
As the transformer is never given $\tau^2$, $S_2/(S_1+1)$ is \emph{not} the Bayes-optimal predictor directly available from
the model's observed inputs. The actual Bayes-optimal predictor must instead average over uncertainty about
the unknown heterogeneity: $\hat\mu_{\mathrm{Bayes}}
=
\mathbb E[
\mu
\mid
x_{1:K},\sigma^2_{1:K}
]$. Equivalently, $\hat\mu_{\mathrm{Bayes}} = \int
\frac{S_2(t)}{S_1(t)+1}
\,
p(t\mid x_{1:K},\sigma^2_{1:K})
\,dt$, where $t$ denotes a candidate value of the heterogeneity, $S_1(t) = \sum_{i=1}^K
\frac{1}{\sigma_i^2+t}$, and $S_2(t)
=
\sum_{i=1}^K
\frac{x_i}{\sigma_i^2+t}$.

The conditional quantity $\frac{S_2}{S_1+1}$ is nevertheless useful because the simulator knows the true $\tau^2$. We
therefore use it only as an \emph{oracle conditional-Bayes reference}: it tells us how well prediction could perform if the hidden heterogeneity were known. It is not supplied to the transformer and is not the model's available Bayes-optimal solution. See Appendix~\ref{app:bayes_posterior} for the full derivation.

\subsection{Distribution shift}
The total variance of study $i$ is $\sigma_i^2+\tau^2$. We summarize which source of uncertainty dominates using the sampling-variance share 
\[
r_i
=
\frac{\sigma_i^2}{\sigma_i^2+\tau^2}.
\]
This quantity lies between zero and one. When $r_i\approx0$,
most uncertainty comes from between-study heterogeneity. When $r_i\approx1$, most uncertainty comes from ordinary study-level sampling variance.

In the in-distribution regime, the median value is approximately $r_{\mathrm{ID}}=0.31$, so heterogeneity is usually the dominant contribution to total variance. In the OOD regime, which draws $s_i^2\sim\mathrm U(10,20)$, $n_i\sim\mathrm U\{2,\ldots,19\}$ and $\tau^2\sim\mathrm U(0,1.5)$ with all other draws as above, the median increases to $r_{\mathrm{OOD}}=0.71$, so sampling variance becomes the dominant contribution.

The distribution shift therefore changes the relative source of uncertainty governing the correct inverse-variance weights. A model that has learned a regime-specific shortcut,
accurate only when $\tau^2$ dominates, can consequently fail when
$\sigma_i^2$ becomes dominant, whereas a model implementing the appropriate
weighting computation should adapt its weights to the changed balance.
\subsection{Heterogeneity baseline}

The simulator knows $\tau^2$; the transformer does not. A classical baseline for estimating heterogeneity from the observed study estimates is the
DerSimonian--Laird estimator,
\[
\hat\tau^2_{\mathrm{DL}}
=
\max\left\{
0,\,
\frac{Q-(K-1)}{C}
\right\},
\]
where $Q$ measures excess dispersion of the study estimates around their
fixed-effect weighted mean and $C$ is a normalization determined by the
fixed-effect study weights.

Since $\hat\tau^2_{\mathrm{DL}}$ is computable from the model inputs, a probe that predicts $\tau^2$ may simply reproduce information already available through this
standard estimator. The full-information features used for the estimated reference points therefore include $\hat\tau^2_{\mathrm{DL}}$ (Appendix~\ref{app:endpoint-features}).

\subsection{Training and checkpoints}
\label{app:training}
Models are causal pre-norm transformers at five sizes, from 3.3k to 794k parameters (Table~\ref{tab:model-size-hyperparameters}), trained with mean-squared error at every prefix. Checkpoints are retained at epochs $\{0,1,2,5,10,20,50,100,\ldots\}$ for training-trajectory analyses. Training stops when validation MSE, averaged over all prefixes, is within $2\%$ of the oracle conditional-Bayes risk averaged the same way (Appendix~\ref{app:analytic_mse}), or after $300$ epochs without validation improvement; the oracle risk uses the true $\tau^2$ and is not attainable from the model's inputs. Of the 25 random-effects runs, 24 stopped on patience. Model predictions are closer to the oracle conditional posterior mean than to the oracle IVW estimator. Behavioral comparisons with the IVW and conditional-Bayes estimators are performed on held-out meta-analyses.

\begin{table*}[t]
\centering
\begin{tabular}{lrrrrrr}
\toprule
Model size & $d$ & $h$ & $L$ & Parameters & Peak learning rate & Epoch budget \\
\midrule
xs & 16  & 2 & 1 & 3,345   & $4\times10^{-4}$ & 3,300 \\
s  & 32  & 4 & 2 & 25,537  & $4\times10^{-4}$ & 3,300 \\
m  & 64  & 4 & 2 & 100,225 & $5\times10^{-5}$ & 3,000 \\
l  & 64  & 4 & 4 & 200,193 & $5\times10^{-5}$ & 3,500 \\
xl & 128 & 8 & 4 & 793,601 & $5\times10^{-5}$ & 3,500 \\
\bottomrule
\end{tabular}
\caption{Transformer model configurations. $d$ is the residual-stream width, $h$ the number of attention heads (head width $d/h$), and $L$ the number of transformer blocks; the block MLP has hidden width $4d$. Parameter counts include the input embedding, all blocks, and the scalar readout.}
\label{tab:model-size-hyperparameters}
\end{table*}

\section{Decoder bank}
\label{sec:probe-bank}

All reported regression-probe scores are computed using one canonical bank (Table~\ref{tab:canonical-probe-bank}) of candidate decoders. For the controlled meta-analysis, all probes and reference points are evaluated at the final position, using the final-token representation. Decoder type and hyperparameters are selected using the selection split based on $R^2$, after which the selected decoder is evaluated once on the  evaluation split. Every probe uses disjoint training, selection, and evaluation rows. Categorical decoders are chosen by McFadden $R^2$, but reported and normalized into $U_H$, by accuracy. We report standard $R^2$ on this held-out split and retain negative values rather than clipping them to zero.

\[
\mathcal{M}
=
\{
\text{GBM},
\text{shallow GBM},
\text{MLP},
\text{RidgeCV}
\}.
\]

\begin{table}[ht]
\centering
\small
\begin{tabular}{p{0.19\linewidth}p{0.72\linewidth}}
\toprule
Decoder & Fixed specification \\
\midrule
GBM &
\texttt{HistGradientBoostingRegressor} with 400 maximum iterations,
learning rate $0.06$, unrestricted depth, early stopping, and a
$0.15$ internal validation fraction. \\[2mm]

Shallow GBM &
\texttt{HistGradientBoostingRegressor} with 600 maximum iterations,
learning rate $0.04$, maximum depth $3$, early stopping, and a
$0.15$ internal validation fraction. \\[2mm]

MLP &
Standardization followed by an MLP with hidden widths
$(256,128)$, $L_2$ penalty $10^{-4}$, 600 maximum iterations,
early stopping, and patience 25. \\[2mm]

Linear &
Standardization followed by RidgeCV with $
\alpha\in\operatorname{logspace}(-6,4,21).$
\\
\bottomrule
\end{tabular}
\caption{The decoder bank. These definitions are fixed across
domains, representations, targets, and random seeds.}
\label{tab:canonical-probe-bank}
\end{table}

\section{Bayes posterior mean conditional on heterogeneity}
\label{app:bayes_posterior} 
Consider a prefix containing the first $k$ studies. Conditional on the overall effect $\mu$, the heterogeneity $\tau^2$, and the sampling variances $\sigma_{1:k}^2$, the reported effects are independent and satisfy $x_i\mid\mu,\tau^2,\sigma_i^2\sim\mathcal N(\mu,\sigma_i^2+\tau^2)$. Define the precision weights and their cumulative sufficient statistics by 
\[ w_i=\frac{1}{\sigma_i^2+\tau^2},\qquad S_1^{(k)}=\sum_{i=1}^{k}w_i,\qquad S_2^{(k)}=\sum_{i=1}^{k}w_i x_i. \]

Ignoring factors that do not depend on $\mu$, the likelihood is 
\[ p(x_{1:k}\mid\mu,\tau^2,\sigma_{1:k}^2)\propto\exp\!\left[-\frac12\sum_{i=1}^{k}w_i(x_i-\mu)^2\right]\propto\exp\!\left[-\frac12\left(S_1^{(k)}\mu^2-2S_2^{(k)}\mu\right)\right]. \] 

Combining this likelihood with the prior $\mu\sim\mathcal N(0,1)$ gives 

\[ \begin{aligned} p(\mu\mid x_{1:k},\tau^2,\sigma_{1:k}^2)
&\propto p(x_{1:k}\mid\mu,\tau^2,\sigma_{1:k}^2)\,p(\mu)\\ &\propto\exp\!\left[-\frac12\left(S_1^{(k)}\mu^2-2S_2^{(k)}\mu\right)\right]\exp\!\left(-\frac12\mu^2\right)\\ &=\exp\!\left[-\frac12\left(\left(S_1^{(k)}+1\right)\mu^2-2S_2^{(k)}\mu\right)\right]\\ 
&\propto\exp\!\left[-\frac12\left(S_1^{(k)}+1\right)\left(\mu-\frac{S_2^{(k)}}{S_1^{(k)}+1}\right)^2\right]. 
\end{aligned} \] 
Therefore, 
\[ \mu\mid x_{1:k},\tau^2,\sigma_{1:k}^2\sim\mathcal N\!\left(\frac{S_2^{(k)}}{S_1^{(k)}+1},\frac{1}{S_1^{(k)}+1}\right). \]

Under squared-error loss, the posterior mean uniquely minimizes the conditional expected loss. Hence, when $\tau^2$ is known, the Bayes estimator is \[ \hat\mu_{\mathrm B}^{(k)}(\tau^2)=\arg\min_{a}\mathbb E\!\left[(a-\mu)^2\mid x_{1:k},\tau^2,\sigma_{1:k}^2\right]=\mathbb E[\mu\mid x_{1:k},\tau^2,\sigma_{1:k}^2]=\frac{S_2^{(k)}}{S_1^{(k)}+1}. \] In the actual prediction task, $\tau^2$ is not observed. The fully Bayes-optimal estimator therefore averages this conditional posterior mean over the posterior distribution of $\tau^2$: \[ \hat\mu_{\mathrm{Bayes}}^{(k)}=\mathbb E[\mu\mid x_{1:k},\sigma_{1:k}^2]=\int_{0}^{0.5}\frac{S_2^{(k)}(t)}{S_1^{(k)}(t)+1}\,p(t\mid x_{1:k},\sigma_{1:k}^2)\,dt, \] where $S_1^{(k)}(t)=\sum_{i=1}^{k}(\sigma_i^2+t)^{-1}$ and $S_2^{(k)}(t)=\sum_{i=1}^{k}x_i(\sigma_i^2+t)^{-1}$. Thus, $S_2^{(k)}/(S_1^{(k)}+1)$ is an oracle Bayes reference that assumes access to the true heterogeneity, whereas the model must approximate the marginal posterior mean by inferring $\tau^2$ from the observed studies. We do not compute $\hat\mu_{\mathrm{Bayes}}$; all behavioral references in this paper are oracle quantities.

\section{Expected mean-squared error of the IVW and Bayes estimators}
\label{app:analytic_mse}  
We derive the expected mean-squared error of the oracle IVW and oracle conditional-Bayes estimators (both using the true $\tau^2$) derived in Appendix~\ref{app:bayes_posterior}. At a prefix containing the first $k$ studies, define $w_i=1/(\sigma_i^2+\tau^2)$, $S_1^{(k)}=\sum_{i=1}^{k}w_i$, and $S_2^{(k)}=\sum_{i=1}^{k}w_i x_i$. Conditional on $\mu$, $\tau^2$, and the reported sampling variances $\sigma_{1:k}^2$, each observation satisfies $x_i\sim\mathcal N(\mu,1/w_i)$, so $S_2^{(k)}\sim\mathcal N(S_1^{(k)}\mu,S_1^{(k)})$. Hence, the weighted sum can be written as 
\[ S_2^{(k)}=S_1^{(k)}\mu+\varepsilon_k,
\qquad \varepsilon_k\mid S_1^{(k)}\sim\mathcal N\!\left(0,S_1^{(k)}\right). \]

The inverse-variance-weighted estimator is 
\[ \hat\mu_{\mathrm{IVW}}^{(k)}:=\frac{S_2^{(k)}}{S_1^{(k)}}=\mu+\frac{\varepsilon_k}{S_1^{(k)}}. \] 

Its estimation error is therefore $\varepsilon_k/S_1^{(k)}$, and its conditional mean-squared error is

\[
\begin{aligned}
\mathbb E\!\left[\left(\hat\mu_{\mathrm{IVW}}^{(k)}-\mu\right)^2 \,\middle|\, S_1^{(k)}\right]
&=
\mathbb E\!\left[\left.\left(\frac{\varepsilon_k}{S_1^{(k)}}\right)^2 \,\right|\, S_1^{(k)}\right] \\
&=
\frac{1}{\left(S_1^{(k)}\right)^2}
\mathbb E\!\left[\varepsilon_k^2 \mid S_1^{(k)}\right] \\
&=
\frac{\operatorname{Var}\!\left(\varepsilon_k \mid S_1^{(k)}\right)}
{\left(S_1^{(k)}\right)^2} \\
&=
\frac{S_1^{(k)}}{\left(S_1^{(k)}\right)^2}
=
\frac{1}{S_1^{(k)}}.
\end{aligned}
\]

For the prior $\mu\sim\mathcal N(0,1)$, the Bayes posterior mean (from Appendix~\ref{app:bayes_posterior}) is 
\[ \hat\mu_{\mathrm B}^{(k)}=\frac{S_2^{(k)}}{S_1^{(k)}+1}. 
\] 
Substituting $S_2^{(k)}=S_1^{(k)}\mu+\varepsilon_k$ gives 
\[ \hat\mu_{\mathrm B}^{(k)}-\mu=\frac{\varepsilon_k-\mu}{S_1^{(k)}+1}. \] 
The variables $\varepsilon_k$ and $\mu$ are independent, with variances $S_1^{(k)}$ and $1$, respectively. Averaging over both the observation noise and the prior on $\mu$ therefore gives 
\[ \mathbb E\!\left[\left(\hat\mu_{\mathrm B}^{(k)}-\mu\right)^2\middle|S_1^{(k)}\right]=
\frac{\operatorname{Var}(\varepsilon_k)+\operatorname{Var}(\mu)}{(S_1^{(k)}+1)^2}=
\frac{S_1^{(k)}+1}{(S_1^{(k)}+1)^2}=
\frac{1}{S_1^{(k)}+1}. \] 

Thus, at every prefix, 
\[ \mathrm{MSE}_{\mathrm B}^{(k)}=
\frac{1}{S_1^{(k)}+1}<\frac{1}{S_1^{(k)}}=
\mathrm{MSE}_{\mathrm{IVW}}^{(k)}. \] 

The absolute improvement is 
\[ \mathrm{MSE}_{\mathrm{IVW}}^{(k)}-\mathrm{MSE}_{\mathrm B}^{(k)}=
\frac{1}{S_1^{(k)}(S_1^{(k)}+1)}, \] 

while the ratio of the Bayes risk to the IVW risk is 
\[ \frac{\mathrm{MSE}_{\mathrm B}^{(k)}}{\mathrm{MSE}_{\mathrm{IVW}}^{(k)}}=
\frac{S_1^{(k)}}{S_1^{(k)}+1}. 
\] 
The gain is therefore largest when only a few studies have been observed and $S_1^{(k)}$ is small, and it vanishes as the accumulated precision grows. To match the model's all-position training objective, we average these conditional risks over every prefix and every generated meta-analysis: \[ \overline{\mathrm{MSE}}_{\mathrm{IVW}}=\mathbb E_{\,\tau^2,\sigma_{1:K}^2,k}\!\left[\frac{1}{S_1^{(k)}}\right],\qquad \overline{\mathrm{MSE}}_{\mathrm B}=\mathbb E_{\,\tau^2,\sigma_{1:K}^2,k}\!\left[\frac{1}{S_1^{(k)}+1}\right]. \] Evaluating these expectations on the generated evaluation data gives oracle risks $\overline{\mathrm{MSE}}_{\mathrm{IVW}}=0.048$ and $\overline{\mathrm{MSE}}_{\mathrm B}=0.041$. The conditional-Bayes risk lower-bounds the risk attainable without $\tau^2$; the IVW risk does not, because the IVW estimator ignores the prior on $\mu$.

\section{Floor--ceiling derivations and reference-point computation}
\label{app:headroom}

This appendix derives the floor--ceiling normalization used in Section~\ref{sec:theory}. We first specialize the argument to the probe target $\log S_1$, then state the general headroom identity and the two mechanisms by which headroom vanishes, and finally describe the exact computation of the reference points in Appendix~\ref{sec:exact-endpoint-certificate}.

\paragraph{Task-specific definitions and bounds.}
\label{app:task_bounds}

One meta-analysis contains the reported effects and sampling variances
\[
X:=x_{1:K}, \qquad O:=\sigma^2_{1:K}, \qquad I:=(O,X).
\]
Thus, $O$ contains the directly reported variances, $X$ contains the additional evidence carried by the reported effects, and $I$ contains all observed information. The latent variables $(\mu,\tau^2)$ are drawn independently of $O$, with
\[
\tau^2\sim\mathrm U(0,\tau_{\max}), \qquad \mu\sim\mathcal N(0,1),
\]
with $\tau_{\max}=0.5$ in distribution and $\tau_{\max}=1.5$ out of distribution (Appendix~\ref{app:bayes}).
At prefix $k$, the probe target is
\[
T_k:=\log S_1^{(k)}=\log\sum_{i\leq k}\frac{1}{\sigma_i^2+\tau^2}.
\]
The target is therefore a deterministic function of $(O,\tau^2)$. Conditional on $O$, the reported effects $X$ are relevant to $T_k$ only through what they reveal about the latent heterogeneity $\tau^2$. Throughout, we write
\[
\operatorname{MMSE}(T\mid I):=\mathbb E[\operatorname{Var}(T\mid I)].
\]

\begin{proposition}[Probe ceiling]
\label{prop:ceiling}
Let $H=\Phi(I)$ be any representation constructed from the observed inputs, whether produced by a trained or untrained architecture, and let $h$ be any probe. The population $R^2$ of $h(H)$ for a target $T$ satisfies
\[
R^2\leq\rho_I^2:= 1-\frac{\mathbb E[\operatorname{Var}(T\mid I)]}{\operatorname{Var}(T)} = 1-\frac{\operatorname{MMSE}(T\mid I)}{\operatorname{Var}(T)}.
\]
Naturally, $\rho_I^2\leq 1$.
\end{proposition}

\begin{proof}
Because $H=\Phi(I)$, the prediction $h(H)$ cannot use information beyond $I$. The minimum mean-squared error attainable from $I$ is achieved by $\mathbb E[T\mid I]$. Therefore,
\[\operatorname{MSE}(h(H),T)\geq\operatorname{MMSE}(T\mid I),\]
and substitution into the definition of population $R^2$ gives the result.
\end{proof}

\begin{proposition}[Floor and headroom]
\label{prop:floor}
Define
\[
\rho_O^2:=1-\frac{\mathbb E[\operatorname{Var}(T\mid O)]}{\operatorname{Var}(T)} = 1-\frac{\operatorname{MMSE}(T\mid O)}{\operatorname{Var}(T)}.
\]
This is the population $R^2$ attained by the optimal predictor using only the reported sampling variances $O$ and the prior on $\tau^2$. In particular, it performs no inference about heterogeneity from the reported effects $X$. Every $O$-measurable predictor satisfies
\[
R^2\leq\rho_O^2.
\]
For any predictor based on the complete observations,
\[
R^2\leq\rho_O^2+\Delta, \qquad \Delta:=\rho_I^2-\rho_O^2\geq0.
\]

Since $O\subseteq I$, conditioning on $I$ cannot increase the Bayes MMSE, so $\rho_O^2\leq\rho_I^2\leq 1$ and therefore $\Delta\geq 0$.
\end{proposition}

\begin{proof}
The conditional expectation $\mathbb E[T\mid O]$ minimizes MSE among all $O$-measurable predictors, giving the first inequality. The second follows from
\[
\rho_O^2+\Delta=\rho_I^2
\]
and Proposition~\ref{prop:ceiling}.
\end{proof}

\begin{remark}[Evidential meaning of the floor and headroom]
\label{rem:evidential_headroom}
The quantity $\rho_O^2$ is called a floor because it is a floor for evidential interpretation, not because every probe must attain it. A probe must exceed this threshold before its score can provide evidence of information derived from $X$ rather than $O$ alone.

The quantity $\Delta$ is the maximum amount by which a predictor can outperform the optimal predictor that ignores the information about $\tau^2$ contained in $X$. Consequently, only
\[
\max\{0,R^2-\rho_O^2\},
\]
which is at most $\Delta$, can provide evidence that the representation encodes heterogeneity inferred from the reported effects. Since $T$ is measurable with respect to $(O,\tau^2)$, the only relevance of $X$ to $T$ given $O$ is through what $X$ reveals about $\tau^2$.
\end{remark}

\begin{proposition}[Regime dependence of the headroom]
\label{prop:regime}
When sampling variance dominates, $T_k$ is almost entirely predictable from the observed variances alone, whereas when heterogeneity dominates, predicting $T_k$ increasingly requires information about the latent $\tau^2$. Let
\[
w_i:=\frac{1}{\sigma_i^2+\tau^2}, \qquad
S_1^{(k)}:=\sum_{i\leq k}w_i, \qquad
T_k:=\log S_1^{(k)}.
\]

\emph{(i) Sampling-variance-dominated regime ($\sigma_i^2\gg\tau^2$).}
For every $\tau^2\geq0$,
\[
\left|\frac{\partial T_k}{\partial\tau^2}\right|
\leq (\min_i\sigma_i^2)^{-1}.
\]
It follows that
\[
\mathbb E[\operatorname{Var}(T_k\mid O)]
\leq \frac{\tau_{\max}^2}{12}
\mathbb E\!\left[(\min_i\sigma_i^2)^{-2}\right].
\]
Therefore, along any sequence of regimes for which
\[
\frac{\tau_{\max}^2\mathbb E[(\min_i\sigma_i^2)^{-2}]}
{\operatorname{Var}(T_k)}
\to0,
\]
we have
\[
\rho_O^2\to1, \qquad
\Delta\to0, \qquad
\rho_I^2\to1.
\]

\emph{(ii) Heterogeneity-dominated regime ($\tau^2\gg\sigma_i^2$).}
At the final position $K$,
\[
\frac{K}{\tau^2+\max_i\sigma_i^2}
\leq S_1^{(K)}
\leq \frac{K}{\tau^2}.
\]
Hence,
\[
\left|T_K-(\log K-\log\tau^2)\right|
\leq \log\!\left(1+\frac{\max_i\sigma_i^2}{\tau^2}\right).
\]
Consider a sequence of regimes in which
\[
\frac{\max_i\sigma_i^2}{\tau^2}\to0
\]
in probability, and suppose that the squared upper bounds
\[
\left\{
\log^2\!\left(1+\frac{\max_i\sigma_i^2}{\tau^2}\right)
\right\}
\]
are uniformly integrable. Then
\[
T_K-(\log K-\log\tau^2)\to0
\]
in $L^2$. If $\operatorname{Var}(T_K)$ remains bounded away from zero, it follows that
\[
\rho_O^2\to0.
\]
Thus, in this limit, nearly all predictable variance in $T_K$ requires information about the latent heterogeneity.
\end{proposition}

\begin{proof}
For part~(i), by definition,
\[
T_k=\log\sum_{i\leq k}w_i, \qquad
w_i=\frac{1}{\sigma_i^2+\tau^2}.
\]
Therefore,
\[
\left|\frac{\partial T_k}{\partial\tau^2}\right|
=
\left|
\frac{\partial\log\sum_{i\leq k}w_i}{\partial\tau^2}
\right|.
\]
Since
\[
\frac{\partial w_i}{\partial\tau^2}=-w_i^2
\qquad\text{and}\qquad
\frac{\partial\log x}{\partial x}=\frac{1}{x},
\]
we have
\[
\begin{aligned}
\left|\frac{\partial T_k}{\partial\tau^2}\right|
&=\frac{\sum_{i\leq k}w_i^2}{\sum_{i\leq k}w_i}\\
&=\sum_{i\leq k}\frac{w_i}{\sum_{j\leq k}w_j}w_i\\
&\leq\max_iw_i\\
&\leq(\min_i\sigma_i^2)^{-1}.
\end{aligned}
\]
The second expression is a weighted average of the $w_i$, and therefore cannot exceed their maximum. Moreover, since $\tau^2\geq0$,
\[
w_i=\frac{1}{\sigma_i^2+\tau^2}\leq\frac{1}{\sigma_i^2}.
\]

Once $O$ is fixed, the map $\tau^2\mapsto T_k$ is therefore Lipschitz with constant
\[
L(O)=(\min_i\sigma_i^2)^{-1}.
\]
For a function $f$ and random variable $V$, if $|f'(V)|\leq L$, then
\[
\operatorname{Var}(f(V))\leq L^2\operatorname{Var}(V).
\]
Therefore,
\[
\operatorname{Var}(T_k\mid O)
\leq (\min_i\sigma_i^2)^{-2}\operatorname{Var}(\tau^2\mid O).
\]
Since $\tau^2$ is independent of $O$ and
\[
\tau^2\sim\mathrm U(0,\tau_{\max}),
\]
we have
\[
\operatorname{Var}(\tau^2\mid O)
=\operatorname{Var}(\tau^2)
=\frac{\tau_{\max}^2}{12}.
\]
Hence,
\[
\operatorname{Var}(T_k\mid O)
\leq \frac{\tau_{\max}^2}{12}(\min_i\sigma_i^2)^{-2}.
\]
Taking expectations gives
\[
\mathbb E[\operatorname{Var}(T_k\mid O)]
\leq
\frac{\tau_{\max}^2}{12}
\mathbb E\!\left[(\min_i\sigma_i^2)^{-2}\right].
\]

By definition,
\[
\rho_O^2
=
1-
\frac{\mathbb E[\operatorname{Var}(T_k\mid O)]}
{\operatorname{Var}(T_k)}.
\]
Therefore,
\[
0\leq1-\rho_O^2
\leq
\frac{1}{12}
\frac{\tau_{\max}^2\mathbb E[(\min_i\sigma_i^2)^{-2}]}
{\operatorname{Var}(T_k)}.
\]
Under the stated condition,
\[
\frac{\tau_{\max}^2\mathbb E[(\min_i\sigma_i^2)^{-2}]}
{\operatorname{Var}(T_k)}
\to0,
\]
so
\[
\rho_O^2\to1.
\]
Since
\[
0\leq\Delta\leq1-\rho_O^2
\]
and
\[
\rho_O^2\leq\rho_I^2\leq1,
\]
it follows that
\[
\Delta\to0, \qquad \rho_I^2\to1.
\]

For part~(ii), at the final position,
\[
S_1^{(K)}
=
\sum_{i=1}^{K}\frac{1}{\sigma_i^2+\tau^2}.
\]
For every $i$,
\[
\tau^2
\leq \sigma_i^2+\tau^2
\leq \tau^2+\max_j\sigma_j^2.
\]
Taking reciprocals,
\[
\frac{1}{\tau^2+\max_j\sigma_j^2}
\leq
\frac{1}{\sigma_i^2+\tau^2}
\leq
\frac{1}{\tau^2}.
\]
Summing over $i$ gives
\[
\frac{K}{\tau^2+\max_i\sigma_i^2}
\leq
S_1^{(K)}
\leq
\frac{K}{\tau^2}.
\]
Since $\log$ is increasing and $T_K=\log S_1^{(K)}$,
\[
\log K-\log(\tau^2+\max_i\sigma_i^2)
\leq
T_K
\leq
\log K-\log\tau^2.
\]
Now,
\[
\log(\tau^2+\max_i\sigma_i^2)
=
\log\tau^2
+
\log\!\left(1+\frac{\max_i\sigma_i^2}{\tau^2}\right),
\]
and therefore
\[
\log K-\log\tau^2
-\log\!\left(1+\frac{\max_i\sigma_i^2}{\tau^2}\right)
\leq
T_K
\leq
\log K-\log\tau^2.
\]
Hence,
\[
\left|T_K-(\log K-\log\tau^2)\right|
\leq
\log\!\left(1+\frac{\max_i\sigma_i^2}{\tau^2}\right).
\]

If
\[
\frac{\max_i\sigma_i^2}{\tau^2}\to0
\]
in probability, then, by continuity of $\log(1+x)$,
\[
\log\!\left(1+\frac{\max_i\sigma_i^2}{\tau^2}\right)\to0
\]
in probability. The preceding bound therefore gives
\[
T_K-(\log K-\log\tau^2)\to0
\]
in probability.

Convergence in probability alone does not guarantee convergence of the mean squared error. The additional uniform-integrability assumption on
\[
\left\{
\log^2\!\left(1+\frac{\max_i\sigma_i^2}{\tau^2}\right)
\right\}
\]
rules out rare extreme values that would contribute a non-vanishing amount to the expected squared error. Since
\[
\left(T_K-(\log K-\log\tau^2)\right)^2
\leq
\log^2\!\left(1+\frac{\max_i\sigma_i^2}{\tau^2}\right),
\]
the two conditions together imply
\[
\mathbb E\!\left[
\left(T_K-(\log K-\log\tau^2)\right)^2
\right]
\to0.
\]
Equivalently,
\[
T_K-(\log K-\log\tau^2)\to0
\]
in $L^2$.

Let
\[
g(\tau^2):=\log K-\log\tau^2.
\]
Then
\[
T_K-g(\tau^2)\to0
\]
in $L^2$. Since $\tau^2$ is independent of $O$,
\[
\mathbb E[g(\tau^2)\mid O]
=
\mathbb E[g(\tau^2)].
\]
Therefore,
\[
\operatorname{Var}(\mathbb E[T_K\mid O])
=
\operatorname{Var}\!\left(
\mathbb E[T_K-g(\tau^2)\mid O]
\right).
\]
By conditional Jensen's inequality,
\[
\operatorname{Var}(\mathbb E[T_K\mid O])
\leq
\mathbb E\!\left[
(T_K-g(\tau^2))^2
\right]
\to0.
\]
Finally, by the law of total variance and Proposition~\ref{prop:floor},
\[
\rho_O^2
=
1-
\frac{\mathbb E[\operatorname{Var}(T_K\mid O)]}
{\operatorname{Var}(T_K)}
=
\frac{\operatorname{Var}(\mathbb E[T_K\mid O])}
{\operatorname{Var}(T_K)}.
\]
If $\operatorname{Var}(T_K)$ remains bounded away from zero, it follows that
\[
\rho_O^2\to0.
\]
Thus, in this limit, nearly all predictable variance in $T_K$ requires information about the latent heterogeneity.
\end{proof}

\begin{corollary}[Distribution shifts can produce a probe--performance dissociation]
\label{cor:dissociation}
Under a distribution shift from a regime with substantial heterogeneity-dependent headroom to a sampling-variance-dominated regime,
\[
\rho_O^2\to1, \qquad \Delta\to0.
\]
Consequently, probe $R^2$ for $T$ can rise because $O$ alone becomes more predictive, while the maximum amount by which any predictor can benefit from information about $\tau^2$ contained in $X$ shrinks to zero. Whether a particular network's probe score rises, and whether its prediction error increases, must be determined experimentally.
\end{corollary}

\begin{proposition}[Ceilings for $\tau^2$]
\label{prop:tau}
Any probe for $\tau^2$ satisfies
\[
R^2\leq\rho_I^2(\tau^2)
:=
1-\frac{\operatorname{MMSE}(\tau^2\mid I)}
{\operatorname{Var}(\tau^2)}.
\]
\end{proposition}

\begin{proof}
This is Proposition~\ref{prop:ceiling} applied with target $\tau^2$.\end{proof}

\begin{remark}[Why identification of $\tau^2$ deteriorates under the shift]
\label{rem:fisher}
If $\mu$ were known, one study would contribute Fisher information
\[
\frac{1}{2(\sigma_i^2+\tau^2)^2}
\]
about $\tau^2$. Thus, scaling the total variance $\sigma_i^2+\tau^2$ by a factor $c$ would reduce this per-study information by a factor of $c^2$. Since typical total variance is substantially larger under the OOD regime, this scaling provides an intuitive explanation for the reduced identifiability of $\tau^2$.\end{remark}

\section{The general headroom identity}
\label{app:general_headroom}

The task-specific argument is an instance of a more general principle. Let $Z$ be a latent variable, let $T=f(O,Z)$ be a square-integrable target, let $X$ denote additional evidence that may reveal $Z$, and let $I:=(O,X)$. Define
\[
\rho_O^2:=1-\frac{\mathbb E[\operatorname{Var}(T\mid O)]}{\operatorname{Var}(T)}, \qquad
\rho_I^2:=1-\frac{\mathbb E[\operatorname{Var}(T\mid I)]}{\operatorname{Var}(T)}, \qquad
\Delta:=\rho_I^2-\rho_O^2,
\]

and define the Bayes update (change in the Bayesian prediction of $T$ after seeing the additional evidence $X$)
\[
B(I,O):=\mathbb E[T\mid I]-\mathbb E[T\mid O].
\]

\begin{proposition}[Headroom identity]
\label{prop:identity}
The headroom satisfies
\[
\Delta\,\operatorname{Var}(T)
=
\mathbb E\!\left[
\operatorname{Var}\!\left(B(I,O)\mid O\right)
\right].
\]
Thus, the headroom is the conditional variance of the Bayes update about $T$ supplied by $X$ beyond the information already available in $O$.
\end{proposition}

\begin{proof}
By definition,
\[
\Delta\,\operatorname{Var}(T)
=
\mathbb E[\operatorname{Var}(T\mid O)]
-
\mathbb E[\operatorname{Var}(T\mid I)].
\]
Conditional total variance gives
\[
\operatorname{Var}(T\mid O)
=
\mathbb E[\operatorname{Var}(T\mid I)\mid O]
+
\operatorname{Var}(\mathbb E[T\mid I]\mid O).
\]
Taking expectations yields
\[
\mathbb E[\operatorname{Var}(T\mid O)]
=
\mathbb E[\operatorname{Var}(T\mid I)] + \mathbb E\!\left[
\operatorname{Var}\!\left(\mathbb E[T\mid I]\mid O\right)
\right].
\]

Hence,
\[
\Delta\,\operatorname{Var}(T)
=
\mathbb E\!\left[
\operatorname{Var}\!\left(\mathbb E[T\mid I]\mid O\right)
\right].
\]

Since $\mathbb E[T\mid O]$ is constant conditional on $O$, subtracting it does not change the conditional variance, giving the result.
\end{proof}

\begin{corollary}[Two mechanisms for vanishing headroom]
\label{cor:channels}
For each fixed $O$, suppose that $z\mapsto f(O,z)$ is Lipschitz continuous with Lipschitz constant $L(O)$,
\[
|f(O,z)-f(O,z')|\leq L(O)|z-z'|,
\]
and has oscillation bounded by $M$,
\[
\operatorname{osc}_z f(O,\cdot)
:=
\sup_zf(O,z)-\inf_zf(O,z)
\leq M.
\]
Then the headroom can vanish through either of two mechanisms:
\[
\Delta
\leq
\frac{\mathbb E[L(O)^2\operatorname{Var}(Z\mid O)]}
{\operatorname{Var}(T)} \qquad {\text{weak target sensitivity}}
\] 
when the target $T=f(O,Z)$ becomes insensitive to the latent $Z$, and
\[
\Delta
\leq
\frac{M^2\,\operatorname{MI}(X;Z\mid O)}
{2\operatorname{Var}(T)} \quad {\text{weak identification}},
\]
when the additional evidence $X$ contains little information about $Z$ beyond $O$, where $\operatorname{MI}$ denotes conditional mutual information in nats. The first inequality describes \emph{weak target sensitivity}, in which the target becomes insensitive to the latent variable. The second describes \emph{weak identification}, in which the additional evidence becomes uninformative about the latent variable given $O$.
\end{corollary}

\begin{proof}
(i) \textit{Weak target sensitivity:} By the conditional law of total variance (proof of Proposition~\ref{prop:identity}),
\[
\operatorname{Var}(\mathbb E[T\mid I]\mid O)
\leq
\operatorname{Var}(T\mid O),
\]
since the variance of the conditional mean cannot exceed the variance of $T$. Since $T=f(O,Z)$, holding $O$ fixed gives $\operatorname{Var}(T\mid O)=\operatorname{Var}(f(O,Z)\mid O)$. By the Lipschitz assumption,
\[
\operatorname{Var}(f(O,Z)\mid O)
\leq
L(O)^2\operatorname{Var}(Z\mid O).
\]
Therefore, 
\[
\operatorname{Var}(\mathbb E[T\mid I]\mid O)
\leq
L(O)^2\operatorname{Var}(Z\mid O).
\]
Taking expectations and applying Proposition~\ref{prop:identity} gives the first inequality.
\[
\Delta
\leq
\frac{\mathbb E[L(O)^2\operatorname{Var}(Z\mid O)]}
{\operatorname{Var}(T)}.
\] 
Intuitively, if $T$ barely changes when $Z$ changes, learning $Z$ cannot help very much.

(ii) \textit{Weak identification:} Let $\pi_{Z\mid I}$ and $\pi_{Z\mid O}$ denote the posterior beliefs of $Z$ given $I$ and $O$. The oscillation bound gives
\[
|B(I,O)|
\leq
M\,\operatorname{TV}(\pi_{Z\mid I},\pi_{Z\mid O}).
\]
Therefore,
\[
\mathbb E[\operatorname{Var}(B(I,O)\mid O)]
\leq
\mathbb E[B(I,O)^2]
\leq
M^2\mathbb E[
\operatorname{TV}^2(\pi_{Z\mid I},\pi_{Z\mid O})
].
\]
Pinsker's inequality gives
\[
\operatorname{TV}^2(\pi_{Z\mid I},\pi_{Z\mid O})
\leq
\frac{1}{2}
\operatorname{KL}(\pi_{Z\mid I}\Vert\pi_{Z\mid O}),
\]
whose expectation is $\operatorname{MI}(X;Z\mid O)$. Applying Proposition~\ref{prop:identity} completes the proof.
\end{proof}

\section{Synthetic single-cell generative modeling}
\label{app:sc_gen_model}
To test whether the estimation procedure generalizes beyond the controlled meta-analysis setting, we construct a second completely defined generative model based on single-cell RNA sequencing. Because the data-generating process is known, the reference points can be computed exactly and compared with their estimates (Appendix~\ref{sec:exact-endpoint-certificate}).

\paragraph{Generative model.}
Each example is one gene $g$ observed in $K=40$ cells, indexed by $i=1,\dots,K$. The draws are
\begin{table}
\centering
\begin{tabular}{@{}ll@{}}
\toprule
Draw & Meaning \\
\midrule
$\log\lambda_g\sim\mathcal N(\mu_\lambda,\sigma_\lambda^2)$ & mean expression \\
$b_g\sim\mathrm U(0,b_{\max})$ & overdispersion (latent $Z$) \\
$\log d_{gi}\sim\mathcal N(\mu_d,\sigma_d^2)$ & sequencing depth of cell $i$ \\
$\theta_{gi}\sim\mathrm{Gamma}\big(\mathrm{shape}=\lambda_g/b_g,\ \mathrm{scale}=b_g\big)$ & true expression rate \\
$y_{gi}\mid\theta_{gi},d_{gi}\sim\mathrm{Poisson}(d_{gi}\theta_{gi})$ & observed count \\
\bottomrule
\end{tabular}
\end{table}

with $\mu_\lambda=-4.5$, $\sigma_\lambda=1$, $b_{\max}=0.05$, $\mu_d=6.9$ and $\sigma_d=0.7$. The draws of $\lambda_g$, $b_g$ and $d_{g,1:K}$ are mutually independent, and the depths are independent across genes and cells. (In the simulation the lower end of the prior on $b_g$ is $10^{-6}$ rather than $0$, which keeps the Gamma shape finite.) Each seed contains $1500$ genes, and all results use five seeds. At these values the median depth is $e^{6.9}\approx 990$ and the median expected count $d_{gi}\lambda_g$ is $e^{2.4}\approx 11$.

The Gamma draw gives $\mathbb E[\theta_{gi}\mid\lambda_g,b_g]=\lambda_g$ and $\operatorname{Var}(\theta_{gi}\mid\lambda_g,b_g)=\lambda_g b_g$, so true expression rate $\theta_{gi}$ differs from mean expression $\lambda_g$ because of biological cell-to-cell variation ($b_g$), while observed count $y_{gi}$ adds technical counting noise from the finite depth $d_{gi}$. Averaging the Poisson likelihood over $\theta_{gi}$ gives a negative-binomial count,
\[
y_{gi}\mid\lambda_g,b_g,d_{gi}\sim\mathrm{NB}\!\left(\mathrm{mean}=d_{gi}\lambda_g,\ \mathrm{size}=\frac{\lambda_g}{b_g}\right),
\]
with $\operatorname{Var}(y_{gi}\mid\lambda_g,b_g,d_{gi})=d_{gi}\lambda_g+b_g d_{gi}^2\lambda_g$. The normalized count $r_{gi}=y_{gi}/d_{gi}$ therefore has
\[
\operatorname{Var}(r_{gi}\mid\lambda_g,b_g,d_{gi})=\lambda_g\left(\frac{1}{d_{gi}}+b_g\right),
\]
the sum of a technical component $1/d_{gi}$ and a biological component $b_g$. This is the variance structure of the meta-analysis task,
\[
\sigma_i^2 \leftrightarrow \frac{1}{d_{gi}},
\qquad
\tau^2 \leftrightarrow b_g.
\]

\paragraph{Target and information sets.}
For known biological overdispersion $b_g$, cell $i$ has precision $w_{gi}=(d_{gi}^{-1}+b_g)^{-1}$. The probe target is
\[
T_g=\log\sum_{i=1}^{K}\frac{1}{d_{gi}^{-1}+b_g},
\]
the analogue of $\log S_1$ with $d_{gi}^{-1}$ in place of $\sigma_i^2$ and $b_g$ in place of $\tau^2$. The restricted information is the depths alone, $O_g=d_{g,1:K}$; the additional evidence is the observed counts, $X_g=y_{g,1:K}$; and the full information is $I_g=(y_{g,1:K},d_{g,1:K})$. Because $b_g$ is independent of $d_{g,1:K}$, the floor $\rho_O^2$ comes only from how $T_g$ varies with the depths, averaged over the prior on $b_g$; inferring $b_g$ requires the counts. The floor $\rho^2_O$ and ceiling $\rho^2_I$ are the population reference points of Section~\ref{sec:theory} (Propositions~\ref{prop:ceiling} and~\ref{prop:floor}).

\paragraph{Binomial thinning.}
\label{app:sc_thinning}
Retaining a fraction $p\in(0,1]$ of the counts gives $y'_{gi}\sim\mathrm{Bin}(y_{gi},p)$ and $d'_{gi}=p\,d_{gi}$. Because binomial thinning of $\mathrm{Poisson}(d_{gi}\theta_{gi})$ is $\mathrm{Poisson}(p\,d_{gi}\theta_{gi})$, the thinned data follow the same model with depths $p\,d_{gi}$ (so $\log d'_{gi}\sim\mathcal N(\mu_d+\log p,\sigma_d^2)$).

\subsection{Separating the two mechanisms of vanishing headroom} 
\label{app:sepparating}
(Figure~\ref{fig:collapse}). Sequencing depth affects headroom in two ways
at once: shallower data give fewer counts from which to infer the dispersion $b_g$ (weaker identification), and
also make technical noise $1/d_{gi}$ larger relative to $b_g$ in the target, so that $T_g$ depends less on $b_g$
(weaker sensitivity). In real data a single depth $p$ moves both. In the synthetic model we vary them separately
with two depth fractions:
\begin{itemize}
  \item the \emph{evidence fraction} $p_X$ thins only the counts the ceiling sees: the count-based features of
        $\phi_I$ are computed from thinned counts $y'_{gi}\sim\mathrm{Bin}(y_{gi},p_X)$ with depths $p_X d_{gi}$;
  \item the \emph{target fraction} $p_T$ scales only the depths in the target,
        $T_g=\log\sum_i\big((p_T d_{gi})^{-1}+b_g\big)^{-1}$.
\end{itemize}
The floor features $\phi_O$ always use the full depths $d_{g,1:K}$; since $p_T$ is a known constant, they carry
the same information as $p_T d_{g,1:K}$. Both fractions run on a log grid, $p_X\in\{1,\dots,10^{-2}\}$ (4 values)
and $p_T\in\{1,\dots,10^{-4}\}$ (7 values), in steps of $10^{-2/3}$. For scale, at full depth the median technical
variance $1/d_{gi}\approx 0.001$ is small next to $b_g\le 0.05$, whereas at $p=0.01$ it is $\approx 0.1$ and
dominates $b_g$. Each grid value is the mean over five seeds of the estimated headroom
$\hat\rho^2_I-\hat\rho^2_O$, both fitted by the decoder bank on the grid's own 50/20/30 split of the genes.
At full depth this gives 0.87, against 0.89 in Table~\ref{tab:cross-domain-calibration-legacy}, which uses a
different split, and the exact 0.880 (Table~\ref{tab:exact-endpoints}).

\subsection{Real data.}
The four real 10x cohorts \citep{10xgenomics2021pbmc, 10xgenomics2024pbmcdonors, 10xgenomics2022tumors} are cast in the same form. We keep the 2000 most frequently detected genes among those detected in at least 50 cells. The sequencing depth of gene $g$ in cell $i$ is $d_{gi}=\ell_i\,f_g$, where $\ell_i$ is the library size of cell $i$ and $f_g$ the fraction of all counts that belong to gene $g$. The cells are split in half at random: biological overdispersion parameter $b_g$ is replaced by its method-of-moments estimate $\hat b_{\mathrm{MoM}}$ (Appendix~\ref{app:endpoint-features}) from the first half, and the target, $O_g$, $I_g$, the features and the scGPT input use the second half, so the ceiling cannot recover its own target estimate. $K$ is the number of cells in the second half. The features are those of the synthetic model, with $b_{\max}$ set to $1.1$ times the largest $\hat b_{\mathrm{MoM}}$ among training genes.

\section{Exact reference points and validation of the estimation procedure}
\label{sec:exact-endpoint-certificate}

\begin{figure}[!htbp]
\centering
\includegraphics[width=0.55\linewidth]{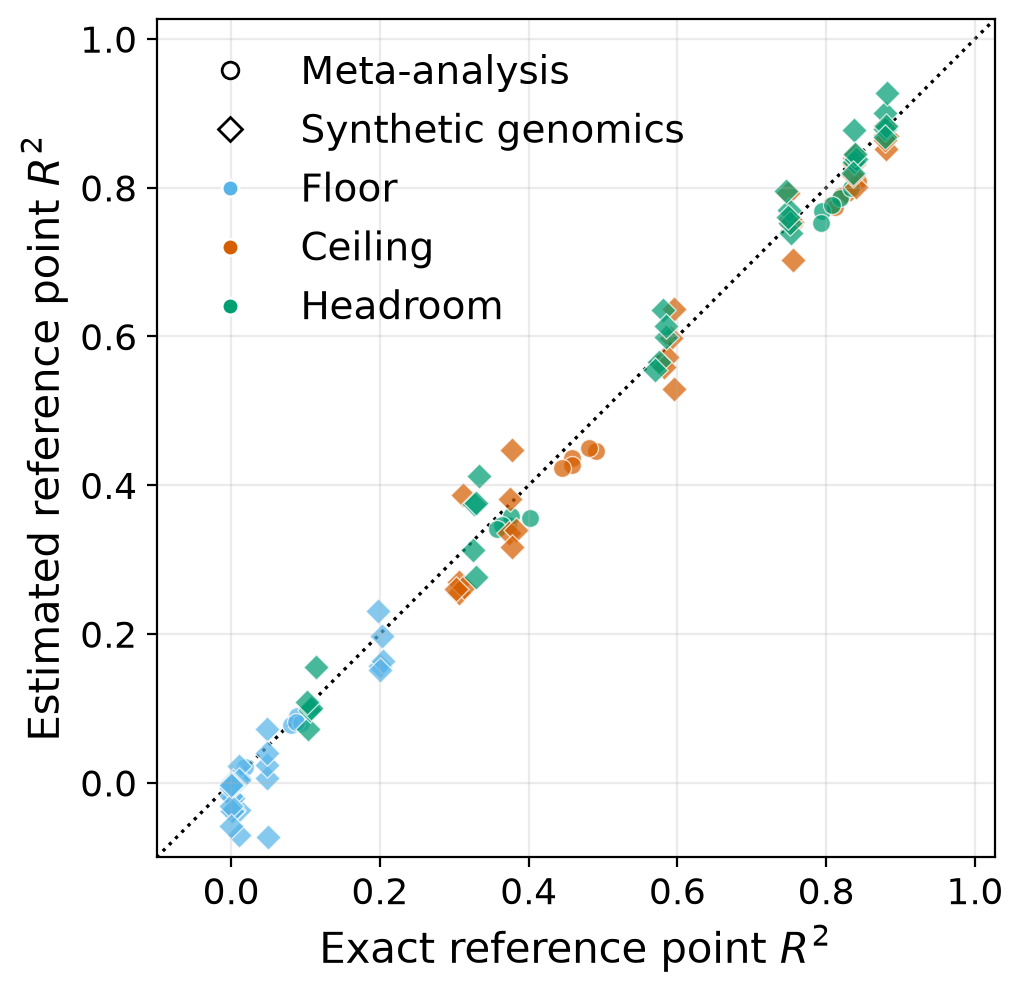}
\caption{\textbf{The estimation procedure recovers exact reference points.} Each marker compares an estimated reference point (decoder bank; vertical axis) with its exact value (horizontal axis), for the meta-analysis (circles) and the synthetic single-cell model (diamonds). Points lie close to the identity line (dotted) across regimes and reference-point types. Slightly negative estimated floors are held-out $R^2$ values where the exact floor is near zero, not negative population floors.}
\label{fig:exact-vs-empirical}
\end{figure}

We test the floor--ceiling construction in two independently specified
generative models: the random-effects meta-analysis model and a synthetic
single-cell RNA-sequencing model.  In both cases the latent distribution and
observation model are known, so the reference points can be computed exactly, independently of the estimation procedure.

\paragraph{Population reference points.}
Let $T$ be the probe target, let $O$ denote the restricted information,
and let $I=(O,X)$ denote the full information after observing the
additional evidence $X$.  Under squared-error loss, the optimal predictors
measurable with respect to these information sets are
\[
m_O(O)=\mathbb E[T\mid O],
\qquad
m_I(I)=\mathbb E[T\mid I].
\]
The corresponding population reference points are
\[
\rho_O^2
=
1-
\frac{\mathbb E[\operatorname{Var}(T\mid O)]}
     {\operatorname{Var}(T)}
=
R^2\!\left(T,m_O(O)\right),
\]
and
\[
\rho_I^2
=
1-
\frac{\mathbb E[\operatorname{Var}(T\mid I)]}
     {\operatorname{Var}(T)}
=
R^2\!\left(T,m_I(I)\right).
\]
Their difference,
\[
\Delta=\rho_I^2-\rho_O^2,
\]
is the headroom supplied by $X$ beyond $O$.

These definitions follow directly from the $L^2$ projection property of
conditional expectation.  Because $O\subseteq I$,
\[
\mathbb E[\operatorname{Var}(T\mid I)]
\leq
\mathbb E[\operatorname{Var}(T\mid O)],
\]
and hence
\[
0\leq \rho_O^2\leq \rho_I^2\leq 1,
\qquad
\Delta\geq 0.
\]
Moreover, for any representation $H=\Phi(I)$ and any probe $h$,
\[
\operatorname{MSE}\!\left(T,h(H)\right)
\geq
\operatorname{MMSE}(T\mid I),
\]
which implies
\[
R^2\!\left(T,h(H)\right)\leq \rho_I^2.
\]
Thus the population ceiling $\rho^2_I$ bounds every representation constructed from $I$, while $\rho_O^2$ is the amount of predictability
already attainable without using the additional evidence $X$.

\paragraph{Exact meta-analysis reference points.}
For the random-effects model,
\[
x_i\mid\mu,\tau^2,\sigma_i^2
\sim
\mathcal N\!\left(\mu,\sigma_i^2+\tau^2\right),
\qquad
\mu\sim\mathcal N(0,1),
\qquad
\tau^2\sim\mathrm U(0,\tau_{\max}),
\]
with $\tau_{\max}=0.5$ in distribution and $\tau_{\max}=1.5$ out of distribution (Appendix~\ref{app:bayes}), we define
\[
O=\sigma_{1:K}^2,
\qquad
I=(\sigma_{1:K}^2,x_{1:K}),
\]
and
\[
T
=
\log S_1
=
\log\sum_{i=1}^{K}\frac{1}{\sigma_i^2+\tau^2}.
\]
Because $\tau^2$ is independent of $O$, the exact restricted-information
predictor is
\[
m_O(O)
=
\int
\log\!\left(
\sum_{i=1}^{K}\frac{1}{\sigma_i^2+\tau^2}
\right)
p(\tau^2)\,d\tau^2.
\]

To obtain the exact full-information predictor, we marginalize the nuisance mean
$\mu$.  For a candidate $\tau^2$, define
\[
\nu_i=\sigma_i^2+\tau^2,
\qquad
A_1=\sum_i \nu_i^{-1},
\qquad
A_x=\sum_i x_i \nu_i^{-1},
\qquad
A_{xx}=\sum_i x_i^2\nu_i^{-1}.
\]
The marginal log likelihood is
\[
\log p(x_{1:K}\mid\tau^2,O)
=
-\frac{1}{2}
\left[
\sum_i\log \nu_i
+\log(1+A_1)
+A_{xx}
-\frac{A_x^2}{1+A_1}
\right]
+\mathrm{const}.
\]
Therefore,
\[
p(\tau^2\mid I)
\propto
p(x_{1:K}\mid\tau^2,O)p(\tau^2),
\]
and
\[
m_I(I)
=
\int
\log\!\left(
\sum_{i=1}^{K}\frac{1}{\sigma_i^2+\tau^2}
\right)
p(\tau^2\mid I)\,d\tau^2.
\]
The one-dimensional prior and posterior integrals are evaluated by numerical
quadrature over $\tau^2$.  The exact floor and ceiling are then obtained by
evaluating $m_O$ and $m_I$ on independent held-out meta-analyses in both
the in-distribution and out-of-distribution regimes.

\paragraph{Exact reference points of the synthetic single-cell model.}
The model, target $T_g$ and information sets $O_g=d_{g,1:K}$ and
$I_g=(y_{g,1:K},d_{g,1:K})$ are defined in Appendix~\ref{app:sc_gen_model}.
Write $T(d_g,b)=\log\sum_{i=1}^{K}(d_{gi}^{-1}+b)^{-1}$.
Because $b_g$ is independent of the depth vector $d_g$,
\[
p(b_g\mid O_g)=p(b_g),
\]
and the exact floor predictor is
\[
m_O(d_g)
=
\int T(d_g,b)\,p(b)\,db.
\]
Its conditional residual variance is
\[
v_O(d_g)
=
\int T(d_g,b)^2p(b)\,db-m_O(d_g)^2.
\]

With the negative-binomial likelihood of the counts,
$y_{gi}\mid\lambda_g,b_g,d_{gi}\sim\mathrm{NB}(\mathrm{mean}=d_{gi}\lambda_g,\ \mathrm{size}=\lambda_g/b_g)$
(Appendix~\ref{app:sc_gen_model}), independent across cells,
\[
p(b_g\mid y_g,d_g)
\propto
p(b_g)
\int
p(y_g\mid\lambda,b_g,d_g)p(\lambda)\,d\lambda.
\]
The exact full-information predictor is
\[
m_I(y_g,d_g)
=
\int
T(d_g,b)\,
p(b\mid y_g,d_g)\,db,
\]
with conditional residual variance
\[
v_I(y_g,d_g)
=
\operatorname{Var}(T_g\mid y_g,d_g).
\]
The exact reference points of the synthetic single-cell model are therefore
\[
\rho_O^2
=
1-\frac{\mathbb E[v_O(O)]}{\operatorname{Var}(T)},
\qquad
\rho_I^2
=
1-\frac{\mathbb E[v_I(I)]}{\operatorname{Var}(T)},
\qquad
\Delta=\rho_I^2-\rho_O^2.
\]
The integrals over $(\lambda,b)$ are evaluated on a fixed numerical grid:
150 cells in $b$ on $(0,b_{\max})$, with edges $b_{\max}u^2$ for $u$ equally
spaced on $[0,1]$ so that the grid is dense near $b=0$, and 50 points in
$\log\lambda$, equally spaced on $\mu_\lambda\pm3.5\,\sigma_\lambda$.
Under binomial thinning (Appendix~\ref{app:sc_thinning}) the same
computation is applied to the thinned $(y',d')$.
The numerator and denominator are computed using the same quadrature rule,
ensuring the total-variance identity
\[
\operatorname{Var}(T)
=
\mathbb E[\operatorname{Var}(T\mid O)]
+
\operatorname{Var}(\mathbb E[T\mid O])
\]
up to numerical precision.

\paragraph{Validation against exact values.}
We independently estimate the same reference points from observed samples without
giving the estimator access to the true latent variable.  Each seed is divided
into disjoint training, model-selection, and evaluation partitions.  Candidate
regressors from the decoder bank are fitted on the training partition, the regressor class is fixed
using the selection partition, and the selected class is refitted on the
combined training and selection data before evaluation.

This produces
\[
\hat\rho_O^2
=
R^2\!\left(T,\hat m_O(O)\right),
\qquad
\hat\rho_I^2
=
R^2\!\left(T,\hat m_I(I)\right),
\qquad
\hat\Delta
=
\hat\rho_I^2-\hat\rho_O^2.
\]
The validation compares
\[
(\rho_O^2,\hat\rho_O^2),\qquad
(\rho_I^2,\hat\rho_I^2),\qquad
(\Delta,\hat\Delta)
\]
for every seed, for both meta-analysis regimes, and for six thinning
fractions $p\in\{1,\,0.40,\,0.16,\,0.063,\,0.025,\,0.01\}$ of the synthetic
single-cell model (Appendix~\ref{app:sc_thinning}). Figure~\ref{fig:exact-vs-empirical} shows the comparison: floors, ceilings and headroom estimated by the estimation procedure lie close to their exact values in both models, and Table~\ref{tab:exact-endpoints} lists the exact reference points.

\begin{table}
\centering
\renewcommand{\arraystretch}{1.2}
\begin{tabular}{llllll}
\toprule
 &  &  & Floor & Ceiling & Headroom \\
Domain & Target & Condition &  &  &  \\
\midrule
\multirow[t]{6}{*}{Meta-analysis} & \multirow[t]{2}{*}{$\log S_1$} & ID & $0.012 \pm 0.004$ & $0.822 \pm 0.014$ & $0.809 \pm 0.017$ \\
 &  & OOD & $0.092 \pm 0.011$ & $0.466 \pm 0.019$ & $0.374 \pm 0.017$ \\
\cline{2-6}
 & \multirow[t]{2}{*}{$\tau^2$} & ID & $0.000 \pm 0.000$ & $0.758 \pm 0.013$ & $0.758 \pm 0.013$ \\
 &  & OOD & $0.000 \pm 0.000$ & $0.404 \pm 0.017$ & $0.404 \pm 0.017$ \\
\cline{1-6}
Synthetic genomics & $T_g$ & Full depth & $0.000 \pm 0.000$ & $0.880 \pm 0.001$ & $0.880 \pm 0.001$ \\
\bottomrule
\end{tabular}
\caption{Exact reference points for the meta-analysis and synthetic-genomics validation experiments (mean $\pm$ sample standard deviation across five seeds).}
\label{tab:exact-endpoints}
\end{table}

\section{Feature engineering for estimated reference points}
\label{app:endpoint-features}

The exact reference points (Appendix~\ref{sec:exact-endpoint-certificate}) condition on the raw information sets $O$ and $I$. The estimated
reference points are instead fitted by the decoder bank (Appendix~\ref{sec:probe-bank}) on fixed-length feature vectors
$\phi_O(O)$ and $\phi_I(I)$:
\[
  \hat\rho_O^2 = R^2\!\big(T,\hat m_O(\phi_O(O))\big), \qquad
  \hat\rho_I^2 = R^2\!\big(T,\hat m_I(\phi_I(I))\big),
\]
where $T$ is the probe target and $\hat m$ the fitted regressor on feature vectors. Each feature vector is a deterministic function of its own information set only. It never uses the latent
variable ($\tau^2$ or $b$) or the target, and $\phi_I$ always contains $\phi_O$. The features are fixed in
advance and are identical across seeds, conditions and sequencing depths.

\paragraph{Notation.} For a vector $a_{1:K}$, $\bar a = K^{-1}\sum_i a_i$;
$\operatorname{sd}(a) = \big(K^{-1}\sum_i (a_i-\bar a)^2\big)^{1/2}$;
$s^2(a) = (K-1)^{-1}\sum_i (a_i-\bar a)^2$;
$a_{(1)}\le\dots\le a_{(K)}$ are the order statistics; $q_\alpha(a)$ is the empirical $\alpha$-quantile with
linear interpolation; $\operatorname{med}(a)=q_{0.5}(a)$. We write $\varepsilon = 10^{-12}$ and
$\operatorname{slog}(a) = \log\max(a,\varepsilon)$.

\paragraph{Meta-analysis.}
Features are computed from all $K=32$ studies of a meta-analysis (the final position). Let
$\ell_i = \log\sigma_i^2$ and $w_i = 1/\sigma_i^2$. Restricted features $\phi_O$ (43 columns), from $O=\sigma^2_{1:K}$:
\begin{enumerate}
  \item sorted log sampling variances $\ell_{(1)},\dots,\ell_{(K)}$ (32);
  \item $\operatorname{mean}(\ell)$, $\operatorname{sd}(\ell)$, $\min_i \ell$, $\max_i \ell$ (4);
  \item quantiles $q_\alpha(\ell)$ for $\alpha\in\{0.10,0.25,0.50,0.75,0.90\}$ (5);
  \item the log total and log squared total fixed-effect precision, $\log\sum_i w_i$ and $\log\sum_i w_i^2$ (2).
\end{enumerate}

Full features $\phi_I$ (114 columns), from $I=(\sigma^2_{1:K},x_{1:K})$:
\begin{enumerate}
  \item $\phi_O$ (43);
  \item the effects $x_1,\dots,x_K$ in study order (32);
  \item $\ell_1,\dots,\ell_K$ in study order, so that each $x_i$ is paired with its own $\sigma_i^2$ (32);
  \item $\operatorname{mean} x$, $\operatorname{sd}(x)$, $\min_i x$, $\max_i x$ (4);
  \item the fixed-effect mean $\hat\mu_{\mathrm{FE}} = \sum_i w_i x_i / \sum_i w_i$, Cochran's
        $Q=\sum_i w_i (x_i-\hat\mu_{\mathrm{FE}})^2$, and the DerSimonian--Laird estimate
        \[
          \hat\tau^2_{\mathrm{DL}} = \max\!\Big\{0,\ \frac{Q-(K-1)}{\sum_i w_i - \sum_i w_i^2/\sum_i w_i}\Big\} \quad (3).
        \]
\end{enumerate}

No information is lost. $T=\log S_1$ is symmetric in the studies and $\tau^2$ is independent of $O$, so
$\mathbb{E}[T\mid O]$ depends on $\sigma^2_{1:K}$ only through the sorted vector in item~1, and
$\mathbb{E}[T\mid O]=\mathbb{E}[T\mid\phi_O(O)]$. Likewise, $I$ is recovered exactly from items~2--3 of $\phi_I$. The
remaining columns are deterministic summaries that ease fitting. The estimated and exact reference points
therefore target the same population quantities. 
\paragraph{Synthetic single-cell model.}
Features are computed separately for each gene using its $K=40$ cells (Appendix~\ref{app:sc_gen_model}). They are constructed from the observed counts $y$ and sequencing depths $d$ after binomial thinning (Appendix~\ref{app:sc_thinning}). Define the normalized count $r_i=\frac{y_i}{d_i}$, the moment estimate of mean expression $\hat\lambda=\operatorname{mean}(r)$,
and the sample variance $s^2=s^2(r).$

The restricted feature vector $\phi_O$ has 8 columns and uses only $O=d_{1:K}$. It contains:
\begin{enumerate}
  \item the number of cells $K$ (1);
  \item $\operatorname{mean}(\log d)$, $\operatorname{sd}(\log d)$, $\min_i \log d$, $\max_i \log d$, and
        $\operatorname{med}(\log d)$ (5);
  \item the log total and log squared total technical variance,
        $\operatorname{slog}\sum_i d_i^{-1}$ and
        $\operatorname{slog}\sum_i d_i^{-2}$ (2).
\end{enumerate}

The full feature vector $\phi_I$ has 27 columns and uses $I=(y_{1:K},d_{1:K})$. We first define three quantities that summarize biological variation beyond Poisson counting noise: (1) the expected Poisson variance, $v=\hat\lambda\,\frac1K\sum_i d_i^{-1}$, where $\hat\lambda$ is the estimated mean expression, $K$ is the number of cells, and $d_i$ is the sequencing depth of cell $i$; (2) the excess variance, $e=\max(0,s^2-v)$, where $s^2$ is the sample variance of the normalized counts $r_i=y_i/d_i$; and (3) the method-of-moments estimate of biological dispersion, $\hat b_{\mathrm{MoM}}=e/\max(\hat\lambda,\varepsilon)$, where $\varepsilon$ is a small constant used to avoid division by zero. 

These definitions follow from $\operatorname{Var}(r_i)=\lambda/d_i+\lambda b$, where $\lambda/d_i$ is the variance expected from Poisson counting noise and $\lambda b$ is the additional variance due to biological dispersion $b$. We also define the depth-weighted mean $\bar r_d=\sum_i y_i/\sum_i d_i$, and the weighted dispersion statistic $Q=\sum_i d_i(r_i-\bar r_d)^2$.

The 27 features in $\phi_I$ are:
\begin{enumerate}
  \item the 8 restricted features in $\phi_O$ (8);

  \item two sparsity features: the detection rate $K^{-1}\sum_i \mathbb{1}[y_i>0]$, which is the fraction of cells with a nonzero count, and $\log(1+\max_i y_i)$, which summarizes the largest observed count (2);

  \item three variance features: $\log(s^2+\varepsilon)$ for the total observed variance of the normalized counts, $\log(v+\varepsilon)$ for the variance expected from Poisson counting noise, and $\log(e+\varepsilon)$ for the excess variance beyond that noise (3);

  \item two representations of the moment estimate of biological dispersion $b$: $\log(\hat b_{\mathrm{MoM}}+\varepsilon)$ and $\hat b_{\mathrm{MoM}}$ (2);

  \item four overdispersion features: $Q$; $\log\!\big(1+\max\{Q-(K-1),0\}\big)$, which records the amount by which $Q$ exceeds the reference value $K-1$; the dispersion ratio $s^2/\max(v,\varepsilon)$, which compares observed variance with the variance expected from Poisson noise; and the normalized-count spread $q_{0.9}(r)-q_{0.1}(r)$, the difference between the 90th and 10th percentiles of $r$ (4);

  \item seven profile-likelihood features. We evaluate the negative-binomial log-likelihood at seven candidate dispersion values $b_j=j\,b_{\max}/7$, for $j=1,\dots,7$. We fix the mean-expression parameter at $\tilde\lambda=\max(\hat\lambda,\varepsilon)$, define the expected count for cell $i$ as $m_i=d_i\tilde\lambda$, and define the corresponding negative-binomial shape parameter as $\kappa_j=\tilde\lambda/b_j$. For each candidate $b_j$, the log-likelihood is $L_j=\sum_i[\ln\Gamma(y_i+\kappa_j)-\ln\Gamma(\kappa_j)-\ln\Gamma(y_i+1)+\kappa_j\log\frac{\kappa_j}{\kappa_j+m_i}+y_i\log\frac{m_i}{\kappa_j+m_i}]$. This is the likelihood from Appendix~\ref{app:sc_gen_model}, with $\lambda$ fixed at $\tilde\lambda$. We use the seven relative likelihood values $L_j-\max_k L_k$ as features, so that the best grid value has value zero and the others indicate how much less likely they are (7);

  \item one likelihood-based estimate of $b$: the grid value with the largest likelihood, $b_{j^\star}$, where $j^\star=\arg\max_j L_j$ (1).
\end{enumerate}

For the synthetic model, $b_{\max}=0.05$, which is the upper bound of the prior on biological overdispersion $b$. For the real cohorts, $b_{\max}$ is set to $1.1$ times the largest $\hat b_{\mathrm{MoM}}$ among the training genes (Appendix~\ref{app:sc_gen_model}).

This construction parallels the meta-analysis setting. The restricted features $\phi_O$ summarize the technical variances $1/d_i$, which play the same role as the sampling variances $\sigma_i^2$. The full features $\phi_I$ additionally include information about biological dispersion, analogous to adding $Q$ and $\hat\tau^2_{\mathrm{DL}}$ in the meta-analysis.

The analogy for $Q$ is only approximate. Here, $Q=\bar r_d\,\chi^2_{\mathrm{P}}$, where $\chi^2_{\mathrm{P}}=\sum_i (y_i-d_i\bar r_d)^2/(d_i\bar r_d)$ is Pearson's statistic. Therefore $Q$ is not standardized, and comparing $Q$ directly with $K-1$ in item~5 is only a heuristic rather than a calibrated DerSimonian--Laird statistic. Standardized information about excess variation is still available to the decoder through $s^2/v$ and $\hat b_{\mathrm{MoM}}$.

Both feature vectors discard some information. Unlike the meta-analysis setting, $\phi_O$ keeps only 8 summaries of the $K$ sequencing depths, while $\phi_I$ does not include the raw vectors $y_{1:K}$ and $d_{1:K}$. Therefore, the best predictors based on $\phi_O$ and $\phi_I$ can fall below the exact population quantities $\rho_O^2$ and $\rho_I^2$, which condition on the full raw data. Figure~\ref{fig:exact-vs-empirical} compares these estimated and exact values.

The full feature vector $\phi_I$ nevertheless contains two direct estimates of the latent biological dispersion $b$: the moment estimator $\hat b_{\mathrm{MoM}}$ and the likelihood estimator $b_{j^\star}$. These features give the estimated reference points at all six thinning fractions in Figure~\ref{fig:exact-vs-empirical}, the full-depth ($p=1$) floor and ceiling in Figure~\ref{fig:calibration} and Table~\ref{tab:cross-domain-calibration-legacy}, and the headroom grid in Figure~\ref{fig:collapse}.
\section{Sensitivity to the choice of \texorpdfstring{$(O,I,T)$}{(O, I, T)}}
\label{app:contrast_sensitivity}

The quantities used for normalization in Section~\ref{sec:theory} are conditional on the choice of $(O,I,T)$. In the controlled meta-analysis and synthetic single-cell experiments, the generative model identifies a natural distinction between restricted and full information. In the observational chemistry and neuroscience applications, however, there is no unique decomposition of the measurements into $O$ and $I$, because changing the choice of $(O,I,T)$ changes the scientific estimand itself.

We therefore evaluate several prespecified, scientifically interpretable contrasts in each observational domain to determine whether the qualitative interpretation of a probe depends strongly on a particular placement of the boundary between restricted and additional information. Whenever possible, we hold the target ($T$) and full information $I$ fixed and progressively enrich the restricted baseline,
\[
\sigma(O_1)\subseteq\sigma(O_2)\subseteq\sigma(O_3)\subseteq\sigma(O_4)\subseteq\sigma(I).
\]
Population monotonicity then implies
\[
\rho_{O_1}^2\leq\rho_{O_2}^2\leq\rho_{O_3}^2\leq\rho_{O_4}^2\leq\rho_I^2,
\qquad
\Delta(O_1,I)\geq\Delta(O_2,I)\geq\Delta(O_3,I)\geq\Delta(O_4,I),
\]
although estimated reference points need not satisfy these inequalities.

\subsection{Chemistry}
\label{app:chemistry_setup}

\paragraph{Dataset and experimental setting.}
We use part~1 of the Raman-ChEMBL database, which provides molecular composition,
physicochemical descriptors, three-dimensional molecular geometry, and
quantum-chemical orbital energies computed with density functional theory (PBE/6-31G). We draw a fixed random
sample of 50{,}000 molecules, keep the 48{,}931 that pass structural validation, and use 30{,}000 of these per
seed. The analysis unit is one molecule. Identical
SMILES are assigned to the same data split so that duplicate molecular identities
do not occur across fine-tuning, validation, probe-fitting, and final-evaluation
sets.

The experiment is useful for probe normalization because molecular orbital energies
are partly predictable from relatively coarse molecular properties such as
composition, charge, and global shape, while detailed atomic geometry provides
additional information. We therefore ask whether TabPFN hidden representations
contain information about orbital energies beyond what was already predictable
from progressively richer non-geometric or coarse molecular summaries. TabPFN hidden-state probes use a 64-component PCA fitted on training rows only. All other probes, floors and ceilings use full features.

\paragraph{Prediction target.}
We analyze the two frontier orbital energies separately,
\[
T^{\mathrm{HOMO}} = E_{\mathrm{HOMO}},
\qquad
T^{\mathrm{LUMO}} = E_{\mathrm{LUMO}},
\]
where $E_{\mathrm{HOMO}}$ is the energy of the highest occupied molecular orbital
and $E_{\mathrm{LUMO}}$ is the energy of the lowest unoccupied molecular orbital.
Both are measured in Hartree.

These two quantities also determine the HOMO--LUMO gap,
\[
E_g =
E_{\mathrm{LUMO}}-E_{\mathrm{HOMO}}.
\]
TabPFN is fine-tuned to predict $E_g$, but HOMO and LUMO are not individually
used as fine-tuning targets. They are subsequently used only as targets of the
representation audit.

\paragraph{Restricted molecular information.}
The restricted information $O$ is constructed from progressively richer
fixed-size molecular summaries. We consider four nested specifications for restricted information,
\[
O_1\subset O_2\subset O_3\subset O_4.
\]

The first contains elemental composition and total formal charge,
\[
O_1 =
\left(
\text{element counts},
Q
\right),
\]
where $Q$ denotes total formal charge. This 18-dimensional representation
describes which atomic species are present and their abundances, but does not
specify how the atoms are arranged in space.

The second specification additionally contains three rotational constants,
\[
O_2 =
\left(
O_1,
R_x,R_y,R_z
\right).
\]
These quantities summarize the distribution of molecular mass around the
principal axes and therefore provide coarse information about molecular size,
mass distribution, and global shape without specifying individual atomic
distances.

The third specification additionally includes the total dipole magnitude,
\[
O_3 =
\left(
O_2,
D_{\mathrm{total}}
\right).
\]
The total dipole provides a coarse summary of molecular charge separation and
therefore introduces global electrostatic information beyond composition and
shape.

Finally, the richest specification is
\[
O_4 =
\left(
O_3,
\alpha_{\mathrm{iso}},
R_{\mathrm{e}}^2,
C_v
\right),
\]
where $\alpha_{\mathrm{iso}}$ is isotropic polarizability,
$R_{\mathrm{e}}^2$ denotes electronic spatial extent, and $C_v$ is heat
capacity. These quantities add information about the response and spatial extent
of the electronic distribution and about the molecule's thermodynamic state. Thus the restricted baseline becomes progressively more informative while
remaining a compact tabular description of the molecule.

\paragraph{Full information $I$.}
The full information is held fixed across all four specifications:
\[
I=
\left(
O_4,
G_{\mathrm{MBTR}}
\right),
\]
where $G_{\mathrm{MBTR}}$ is the unchanged monolithic MBTR geometry block.

The MBTR representation encodes species-resolved pairwise inverse-distance
information. It describes which types of atoms occur at which relative
distances. It is invariant to translation, rotation, and atom ordering. Unlike
the tabular descriptors in the restricted information, it provides a detailed representation of
molecular geometry.

As $O_j$ becomes richer, the variables that $I$ adds become more specific. Above $O_1$, it includes the full physicochemical and geometric
information absent from composition alone, whereas above $O_4$ the remaining increment
is specifically the detailed MBTR geometry block.

The associated headroom,
\[
\hat\Delta_j=\hat\rho^2_I-\hat\rho^2_{O_j},
\]
measures how much additional predictability of the orbital energy remains after
conditioning on the corresponding molecular baseline.

\paragraph{Model representation.}
TabPFN is fine-tuned with LoRA using the fixed full information $I$ to predict
\[
E_g =
E_{\mathrm{LUMO}}-E_{\mathrm{HOMO}}.
\]
Neither $E_{\mathrm{HOMO}}$ nor $E_{\mathrm{LUMO}}$ is supplied as an individual
fine-tuning target. LoRA adapters (rank 8, $\alpha=16$, dropout 0.05) are added to the attention and MLP
projections of every TabPFN~v3 block and to its output projection, and are fitted on molecules disjoint from
those used for probing.

After fine-tuning, hidden representations $H^{(\ell)}$ are extracted from the
audited TabPFN layers. Probes from the decoder bank are then trained separately to predict
$T^{\mathrm{HOMO}}$ and $T^{\mathrm{LUMO}}$ from these representations, producing $R^2_{H^{(\ell)}}$.

Thus, the estimated reference points $\hat\rho^2_{O_j}$ and $\hat\rho^2_I$ correspond to the chosen $(O_j,I,T)$, and $R^2_{H^{(\ell)}}$ measures the orbital-energy information accessible from the model's internal
representation. TabPFN receives $O_4$ and a 99\%-variance PCA of the MBTR block; the ceiling probe uses the raw MBTR, so this PCA does not limit the ceiling. 

\paragraph{What is the scientific question?}
The raw question
\[
\text{``How well can a TabPFN hidden state predict HOMO or LUMO?''}
\]
is insufficient because substantial orbital-energy information may already be
predictable from molecular composition or other coarse physicochemical
properties. The question posed by normalization is instead
\[
\begin{aligned}
&\text{``How should orbital-energy decodability be interpreted relative to what was already} \\
&\text{predictable from a declared molecular summary and what additional predictability} \\
&\text{is available from the full molecular descriptor set $I$?''}
\end{aligned}
\]

For each orbital-energy target and each choice of $(O_j,I,T)$, we therefore compare $\hat\rho^2_{O_j}$, $\hat\rho^2_I$ and $R^2_{H^{(\ell)}}$. The first quantity measures what is predictable from the restricted molecular
description, the second measures what is predictable from the fixed
geometry-aware representation, and their difference defines the estimated
headroom $\hat\Delta_j$. The model probe is interpreted relative to these quantities
rather than on an absolute $R^2$ scale.

The sequence $O_1\subset O_2\subset O_3\subset O_4$ further provides a sensitivity
analysis for the choice of restricted information. $O_1$ asks what the complete
molecular representation adds beyond composition and charge; $O_2$ additionally
conditions on coarse molecular size and shape; $O_3$ further conditions on global
charge separation; and $O_4$ conditions on a rich physicochemical state before
asking what detailed geometry adds. These baseline specifications define distinct estimands. If
the interpretation of the TabPFN probe changes across them, that dependence
identifies which classes of molecular information account for the apparent
decodability.

\textbf{Result.} Most of each raw probe score is already predictable
from the physicochemical baseline $O_4$. Above it, the fine-tuned TabPFN
recovers 69\% (HOMO) and 82\% (LUMO) of the additional predictability
that detailed geometry provides (Table~\ref{tab:cross-domain-calibration-legacy}). Neither orbital
energy was a fine-tuning target, so substantial information about both
is present in the representation even though the model was trained only
on their difference. The interpretation is stable across the nested
baselines $O_1$--$O_4$.

\subsection{Neuroscience}
\label{app:neuroscience_setup}
\paragraph{Related work and motivation.} 
Neuroscience has long used decoding to ask what neural populations represent \citep{quian2009extracting,naselaris2011encoding,glaser2020machine}, and has documented its limits \citep{kriegeskorte2019interpreting,ritchie2019decoding,hebart2018deconstructing}. Two common practices resemble our reference points. First, decoding analyses control for confounds by removing the variance explained by variables that are not of interest \citep{snoek2019control}, and encoding studies split variance between feature sets \citep{deheer2017hierarchical}. Our floor plays this role, with the baseline stated explicitly and chosen by the scientific question. Second, encoding models are compared against a noise ceiling: the part of the neural response that repeats across identical trials, which no model can exceed \citep{hsu2004quantifying,schoppe2016measuring,lagecastellanos2019methods}. Unlike a noise ceiling, which is limited by trial-to-trial variability in the response, the ceiling we define is limited by how much the available information reveals about the target.

\paragraph{Dataset and experimental setting.}
We use the Area~2 dataset of \citet{chowdhury2022area2data}, consisting of neural population recordings from
proprioceptive Area~2 of somatosensory cortex during an arm-reaching task with
externally imposed mechanical perturbations. On perturbation trials, the animal
holds or moves a manipulandum while a brief planar bump is applied through the
interface. The dataset contains simultaneously recorded neural activity together
with hand position, hand velocity, interface force, and the commanded perturbation.
The analysis unit is one perturbation trial.

The experiment is useful for probe normalization because the subsequent mechanical
response is partly predictable from the state of the limb before the perturbation,
while knowledge of the externally imposed bump provides additional information.
We therefore ask whether post-perturbation neural population activity contains
information about the ensuing force response beyond what was already predictable
from the measured pre-bump state.

\paragraph{Prediction target.}
We analyze the two planar force components separately. For component
$c\in\{x,y\}$, the target is the baseline-corrected interface-force response
during the first $100\,\mathrm{ms}$ after bump onset,
\[
T_c
=
\overline{F_c(t)}_{\,0\leq t<100\,\mathrm{ms}}
-
\overline{F_c(t)}_{\,-100\leq t<0\,\mathrm{ms}} .
\]
Thus $T_c$ measures the change in force following the perturbation relative to
the force already present immediately before it. Baseline correction removes
pre-existing load, tonic force, and constant sensor offsets, so the target
reflects the perturbation-evoked mechanical response rather than the absolute
force reading.

\begin{table}
\centering
\begin{tabular}{lllllll}
\toprule
Baseline & Target &
\shortstack{Floor mean \\ $\pm$ SD} &
\shortstack{Ceiling mean \\ $\pm$ SD} &
\shortstack{Headroom mean \\ $\pm$ SD} &
\shortstack{Probe mean \\ $\pm$ SD} &
\shortstack{$U_H$ mean \\ $\pm$ SD} \\
\midrule

\rowcolor{neurocolor}
\multicolumn{7}{l}{\textbf{Neuroscience}} \\

$O_1$ & Force $x$ & $-$0.013 $\pm$ 0.019 & 0.987 $\pm$ 0.009 & 1.000 $\pm$ 0.025 & 0.745 $\pm$ 0.014 & 0.758 $\pm$ 0.019 \\
$O_1$ & Force $y$ & $-$0.002 $\pm$ 0.006 & 0.997 $\pm$ 0.000 & 0.999 $\pm$ 0.006 & 0.573 $\pm$ 0.019 & 0.575 $\pm$ 0.017 \\
$O_2$ & Force $x$ & 0.004 $\pm$ 0.009 & 0.987 $\pm$ 0.009 & 0.983 $\pm$ 0.016 & 0.745 $\pm$ 0.014 & 0.754 $\pm$ 0.018 \\
$O_2$ & Force $y$ & $-$0.013 $\pm$ 0.028 & 0.997 $\pm$ 0.000 & 1.010 $\pm$ 0.028 & 0.573 $\pm$ 0.019 & 0.580 $\pm$ 0.011 \\
$O_3$ & Force $x$ & $-$0.002 $\pm$ 0.019 & 0.987 $\pm$ 0.009 & 0.989 $\pm$ 0.025 & 0.745 $\pm$ 0.014 & 0.755 $\pm$ 0.019 \\
$O_3$ & Force $y$ & $-$0.003 $\pm$ 0.005 & 0.997 $\pm$ 0.000 & 1.000 $\pm$ 0.005 & 0.573 $\pm$ 0.019 & 0.575 $\pm$ 0.019 \\
$O_4$ & Force $x$ & 0.111 $\pm$ 0.023 & 0.987 $\pm$ 0.009 & 0.876 $\pm$ 0.028 & 0.745 $\pm$ 0.014 & 0.724 $\pm$ 0.018 \\
$O_4$ & Force $y$ & 0.088 $\pm$ 0.024 & 0.997 $\pm$ 0.000 & 0.909 $\pm$ 0.024 & 0.573 $\pm$ 0.019 & 0.533 $\pm$ 0.025 \\

\rowcolor{chemcolor}
\multicolumn{7}{l}{\textbf{Chemistry}} \\

$O_1$ & HOMO & 0.378 $\pm$ 0.004 & 0.628 $\pm$ 0.005 & 0.250 $\pm$ 0.007 & 0.555 $\pm$ 0.014 & 0.710 $\pm$ 0.065 \\
$O_1$ & LUMO & 0.458 $\pm$ 0.020 & 0.693 $\pm$ 0.007 & 0.236 $\pm$ 0.014 & 0.660 $\pm$ 0.012 & 0.855 $\pm$ 0.071 \\
$O_2$ & HOMO & 0.381 $\pm$ 0.003 & 0.628 $\pm$ 0.005 & 0.247 $\pm$ 0.003 & 0.555 $\pm$ 0.014 & 0.705 $\pm$ 0.070 \\
$O_2$ & LUMO & 0.460 $\pm$ 0.007 & 0.693 $\pm$ 0.007 & 0.233 $\pm$ 0.011 & 0.660 $\pm$ 0.012 & 0.857 $\pm$ 0.066 \\
$O_3$ & HOMO & 0.383 $\pm$ 0.003 & 0.628 $\pm$ 0.005 & 0.245 $\pm$ 0.005 & 0.555 $\pm$ 0.014 & 0.703 $\pm$ 0.069 \\
$O_3$ & LUMO & 0.457 $\pm$ 0.016 & 0.693 $\pm$ 0.007 & 0.236 $\pm$ 0.012 & 0.660 $\pm$ 0.012 & 0.856 $\pm$ 0.070 \\
$O_4$ & HOMO & 0.391 $\pm$ 0.004 & 0.628 $\pm$ 0.005 & 0.237 $\pm$ 0.005 & 0.555 $\pm$ 0.014 & 0.693 $\pm$ 0.073 \\
$O_4$ & LUMO & 0.504 $\pm$ 0.016 & 0.693 $\pm$ 0.007 & 0.190 $\pm$ 0.012 & 0.660 $\pm$ 0.012 & 0.820 $\pm$ 0.089 \\

\bottomrule
\end{tabular}
\caption{Estimated reference points and probe scores for the nested baselines $O_1\subset O_2\subset O_3\subset O_4$, reported as mean $\pm$ sample standard deviation (SD) across five seeds.}
\label{tab:nested-contrast-results}
\end{table}

\paragraph{Restricted pre-bump information.}
The restricted information $O$ is constructed entirely from measurements before
bump onset. The primary pre-bump window is $[-100,0)\,\mathrm{ms}$, which summarizes the limb state immediately before the perturbation while excluding post-bump information. Averaging within this interval reduces
measurement noise relative to conditioning on a single time point and provides a
local estimate of the state from which the perturbation begins.

We consider progressively richer restricted-information specifications:
\[
\begin{aligned}
O_1 &= 
\left(
\overline{x}_{\mathrm{pre}},
\overline{y}_{\mathrm{pre}}
\right),\\
O_2 &=
\left(
O_1,
\overline{v_x}_{\mathrm{pre}},
\overline{v_y}_{\mathrm{pre}}
\right),\\
O_3 &=
\left(
O_2,
\overline{F_x}_{\mathrm{pre}},
\overline{F_y}_{\mathrm{pre}}
\right).
\end{aligned}
\]
Here each bar denotes the mean over $[-100,0)\,\mathrm{ms}$. These contrasts
progressively condition on hand location, motion, and the force already being
exerted before the bump.

The richest baseline, $O_4$, additionally captures whether these quantities were
changing immediately before perturbation. For any pre-bump variable $q(t)$, we
define
\[
\Delta q_{\mathrm{pre}}
=
\overline{q(t)}_{\,-30\leq t<0\,\mathrm{ms}}
-
\overline{q(t)}_{\,-100\leq t<-70\,\mathrm{ms}} .
\]
We compute this late-minus-early change separately for planar position, velocity,
and force, yielding six additional features:
\[
\Delta x,\;
\Delta y,\;
\Delta v_x,\;
\Delta v_y,\;
\Delta F_x,\;
\Delta F_y .
\]
Thus,
\[
O_4=
\left(
O_3,
\Delta x,\Delta y,
\Delta v_x,\Delta v_y,
\Delta F_x,\Delta F_y
\right).
\]
These features distinguish a stable pre-bump state from one in which the limb or
interface force was already drifting immediately before perturbation, while
remaining entirely pre-perturbation.

\paragraph{Full information $I$.}
The full information is held fixed across all four specifications. It augments
the richest pre-bump state with the experimentally commanded planar perturbation,
\[
I=(O_4,b_x,b_y),
\]
where $(b_x,b_y)$ is the commanded bump vector. Because
$O_1\subset O_2\subset O_3\subset O_4$, every specification satisfies
$\sigma(O_j)\subseteq\sigma(I)$. The ceiling $\hat\rho^2_I$ is therefore fitted
once per target and seed and shared by all four contrasts; only the floor changes
between them.

The associated headroom,
\[
\hat\Delta_j=\hat\rho^2_{I}-\hat\rho^2_{O_j},
\]
measures how much additional predictability of the subsequent force response is
available from knowing the perturbation after conditioning on a specified
description of the pre-bump body state.

\paragraph{Neural representation.}
The representation $H$ is the post-perturbation Area~2 population activity
extracted from the neural response following bump onset. A decoder-bank probe is trained to
predict $T_x$ or $T_y$ from this population representation (spike counts), separately within each recording
session with predictions pooled across sessions, producing
$R_H^2$; like the ceiling, this probe score is shared by all four contrasts.

This setting differs from the artificial-network experiments because $H$ is a
recorded biological population state rather than a deterministic transformation
of the measured $I$. Area~2 activity may contain proprioceptive, biomechanical,
or internally generated information that is absent from our behavioral
measurements.

\paragraph{What is the scientific question?}
The raw question
\[
\text{``How well can post-bump Area~2 activity predict the force response?''}
\]
is insufficient because part of that response may already be predictable from the state of the hand before the perturbation. The question posed by normalization is instead
\[
\begin{aligned}
&\text{``How should neural decodability be interpreted relative to what was already predictable} \\
&\text{before the bump and what additional predictability is provided by knowing the bump?''}
\end{aligned}
\]

For each target and each choice of $(O_j,I,T)$, we therefore compare
\[
\hat\rho^2_{O_j},
\qquad
\hat\rho^2_{I},
\qquad
R^2_H.
\]
The first quantity is the floor from the pre-bump state alone, the second is the estimated full-information reference,
and their difference is the estimated headroom $\hat\Delta_j$. The neural probe
score is interpreted relative to these quantities rather than on an absolute
$R^2$ scale.

The sequence $O_1\subset O_2\subset O_3\subset O_4$ further provides a sensitivity
analysis for the choice of restricted information. $O_1$ asks whether bump
information matters beyond initial hand position alone; $O_2$ additionally
conditions on ongoing motion; $O_3$ conditions on pre-existing interface force;
and $O_4$ further accounts for recent pre-bump trends. These baseline specifications define distinct
estimands. If the interpretation of the neural probe changes across them, that
dependence identifies which aspects of the pre-perturbation state are responsible
for the apparent decodability.

\textbf{Result.} Under $O_4$, floors are near 0.1 and the full-information reference is near
1, so normalized and raw scores nearly coincide: Area~2 activity
captures about 72\% of the perturbation-driven force response along $x$
and 53\% along $y$ under the richest pre-bump baseline $O_4$
(Table~\ref{tab:cross-domain-calibration-legacy}). Because neural activity is not a function of
the measured inputs, the reference is not an upper bound on the probe
score. The interpretation is stable across $O_1$--$O_4$; only $O_4$,
which adds pre-bump trends, raises the floor appreciably.

\section{Normalizing published LLM probes}
\label{app:llm}

\subsection{Common setup}
\label{app:llm_setup}

\textbf{Floors without the model.} In every analysis, $O$ is a simple
description of the input that uses no language model: character
$n$-grams of a name, word and character $n$-grams of a statement or conversation, or
summaries of an Othello move sequence. The floor therefore needs no
access to the model and can be computed for any published result whose
inputs, labels, and splits are released.

\textbf{What the floor measures.} Here $O$ is computed from the input
string itself, so in principle it holds the same information as the
input. What differs is what can be learned from it: the floor is the
best readout that a decoder from the bank (Appendix~\ref{sec:probe-bank}) can learn from $O$
using only the probe's training data. The headroom
above it is what the model adds, through knowledge from its training or
through computation.

\textbf{Metric.} We use $R^2$ for continuous targets and accuracy for
categorical ones. The decoder is chosen on a validation split taken from
the training rows and evaluated once on the test rows. A floor can never truly be below chance; an estimated floor well below chance
means the readout's cues invert outside its training data. We do not
compute $U_H$ in that case and compare raw scores instead.
\begin{table}[!htbp]
\centering
\begin{tabular}{llll}
\toprule
Dataset & Floor (95\% CI) & Probe $R^2$ (7B / 70B) & $U_H$ (7B / 70B) \\
\midrule
World places & 0.61 (0.59--0.62) & 0.88 / 0.91 & 0.70 / 0.77 \\
US places & 0.46 (0.44--0.48) & 0.80 / 0.86 & 0.63 / 0.75 \\
NYC places & 0.33 (0.29--0.37) & 0.22 / 0.36 & $-$0.17 / 0.04 \\
Historical figures & 0.60 (0.59--0.62) & 0.79 / 0.83 & 0.46 / 0.58 \\
Entertainment & 0.16 (0.05--0.28) & 0.79 / 0.89 & 0.75 / 0.86 \\
Headlines & 0.45 (0.43--0.47) & 0.56 / 0.75 & 0.21 / 0.54 \\
\bottomrule
\end{tabular}
\caption{Spelling floors for the space and time datasets of \citet{gurnee2024language}. The floor is the held-out $R^2$ of the selected decoder on character $n$-grams of the entity name; probe scores are the published linear probes on Llama-2-7B / 70B; $U_H$ uses a ceiling of 1, which can only lower a positive $U_H$ (Appendix~\ref{app:llm_setup}).}
\label{tab:space}
\end{table}
\textbf{Ceiling.} We cannot estimate the ceiling here, because predicting
a city's location or a statement's truth from its text needs knowledge
that the probe's training data cannot supply. We set it to 1, the
highest possible value, which can only lower a positive $U_H$.
In Othello the ceiling is exactly 1, since the board follows from the
moves. When a probe scores at or near the floor, we compare the two
directly, which does not depend on the ceiling. The floors themselves
are lower estimates: a stronger readout could raise them and lower
every $U_H$. The two choices push $U_H$ in opposite directions, so we read it as approximate.

The space and time floors use character 2--4-gram counts (Appendix~\ref{app:llm_space}), and the truth and user-attribute floors word 1--2-gram and character 2--4-gram TF-IDF features; in both cases the linear decoder uses the sparse features directly, and the nonlinear decoders use a 1{,}024-component truncated SVD.
\subsection{Space and time}
\label{app:llm_space}

\textbf{Published claim.} \citet{gurnee2024language} train linear ridge
probes on Llama-2 activations at the last token of an entity name to
predict its real-world location (world, US, and New York City places) or
date (historical figures, entertainment, news headlines), and report
high out-of-sample $R^2$ as evidence of spatial and temporal
representations. For New York City, where scores are lowest, they note
that the largest model performs best and suggest that larger models may
form city-level spatial models.

\textbf{Choice of $(O,I,T)$.} $T$ is a coordinate or a date, $I$ is the
entity string exactly as presented to the model, and $O$ is its
character 2--4-gram counts. The floor measures how much of $T$ a readout
can learn from spelling alone, for example from state modifiers, borough
names, or numbers in the string. The original authors raise a related
concern, that a probe could assemble coordinates from simpler features
under supervision, and address it by holding out whole regions; the floor
quantifies one specific form of it.

\textbf{Matching the published evaluation.} We use the published probe
scores (their Table~2: linear probes at 60\% layer depth, Llama-2-7B and
Llama-2-70B) and re-ran one to confirm that our pipeline matches theirs:
for Llama-2-7B on New York City we obtain $R^2=0.218$ against the
published 0.219. We verified that (i) the floor uses exactly the strings
the model received, including rewritten forms such as ``Iowa's Dallas
County'' and ``$\langle$creator$\rangle$'s $\langle$title$\rangle$'',
matching token for token on all 186,669 rows; (ii) rows and splits follow
the released train--test assignment, which groups New York City
locations by complex and entertainment items by creator; and (iii)
targets are in their original units, with $R^2$ computed as in the
original evaluation. Ridge regression is fitted on the $n$-gram counts,
and gradient boosting and an MLP on a 1{,}024-component truncated SVD of
them. Confidence intervals come from 1,000 bootstrap resamples of the
test set, resampling complexes for New York City and creators for
entertainment.

\textbf{Result} (Table~\ref{tab:space}). For world and US places,
spelling accounts for half to two-thirds of the raw $R^2$, yet the probes
recover 63--77\% of the remaining headroom, so most of what remains
reflects information beyond the names. For New York City, spelling
reaches $R^2=0.33$; the Llama-2-70B score (0.36) lies within the floor's
confidence interval and the Llama-2-7B score (0.22) lies below it. At
this scale the reported scores cannot separate a city-level spatial
representation from lexical cues, and the better relative performance
of larger models is also consistent with better use of such cues. This
agrees with the original block-holdout analysis, in which probes
generalized to held-out boroughs only slightly better than chance. For
temporal targets, the probes recover most of the headroom for
entertainment ($U_H\ge 0.75$), more than half for historical figures with
the larger model (0.58; 0.46 for the smaller), and about half for
headlines with the larger model (0.54; 0.21 for the smaller).
\begin{table*}[!htbp]
\centering
\small
\setlength{\tabcolsep}{4pt}
\renewcommand{\arraystretch}{1.08}

\begin{tabularx}{\textwidth}{Xcccccccccc}
\toprule
\makecell{Floor\\features} &
Decoder &
\makecell{Floor acc.\\(95\% CI)} &
\makecell{Linear\\floor acc.} &
\makecell{Chance\\acc.} &
\makecell{Occupied\\acc.} &
\makecell{Probe acc.\\(95\% CI)} &
\makecell{Linear probe\\acc.} &
$U_H$ &
\makecell{$U_H$ linear\\probe} \\
\midrule

\rowcolor{neurocolor}
\multicolumn{10}{l}{\textbf{Black/White/Empty}} \\

$O_1$: position
& GBM
& 0.62 (0.61--0.62)
& 0.56
& 0.46
& 0.39
& 0.89 (0.89--0.90)
& 0.74
& 0.72
& 0.33 \\

$O_2$: recency
& GBM
& 0.66 (0.65--0.67)
& 0.60
& 0.46
& 0.47
& 0.89 (0.89--0.90)
& 0.74
& 0.68
& 0.24 \\

$O_3$: placement
& GBM
& 0.80 (0.80--0.81)
& 0.76
& 0.46
& 0.64
& 0.89 (0.89--0.90)
& 0.74
& 0.45
& $-0.30$ \\

\addlinespace
\rowcolor{chemcolor}
\multicolumn{10}{l}{\textbf{Mine/Theirs/Empty}} \\

$O_1$: position
& MLP
& 0.62 (0.61--0.63)
& 0.56
& 0.46
& 0.40
& 0.98 (0.98--0.99)
& 0.99
& 0.96
& 0.96 \\

$O_2$: recency
& GBM
& 0.66 (0.66--0.67)
& 0.57
& 0.46
& 0.47
& 0.98 (0.98--0.99)
& 0.99
& 0.95
& 0.96 \\

$O_3$: placement
& GBM
& 0.81 (0.81--0.82)
& 0.80
& 0.46
& 0.65
& 0.98 (0.98--0.99)
& 0.99
& 0.92
& 0.92 \\

\bottomrule
\end{tabularx}

\caption{Floors, probe accuracy and headroom-normalized decodability for the
Black/White/Empty and Mine/Theirs/Empty targets under nested floor
feature sets $O_1 \subset O_2 \subset O_3$. The ceiling is exactly 1,
since the board is a deterministic function of the move sequence.
Occupied acc.\ is the floor accuracy on occupied squares only.
$U_H$ is the fraction
of the headroom recovered by the probe.}

\label{tab:othello}
\end{table*}
The entertainment floor is low (0.16), possibly
because the released split holds out entire creators, so spelling cues
learned for one creator's titles cannot transfer to another's. The
entertainment result is robust to the floor's uncertainty: across its interval, $U_H$
lies between 0.71 and 0.78 for Llama-2-7B and between 0.85 and 0.88 for
Llama-2-70B. New York City locations are also grouped, yet their floor
stays high, plausibly because names carry borough names and numbers
shared across complexes. 

\subsection{Othello world model}
\label{app:llm_othello}
\textbf{Published claim.} \citet{li2022emergent} train probes on
Othello-GPT activations to predict the state of each board square, and
report 0.88--0.98 accuracy for nonlinear probes and 0.76--0.79 for linear probes on
a black/white/empty encoding for synthetic tasks, as evidence of an emergent world model.
\citet{nanda2023emergent} show that a mine/theirs/empty encoding, relative
to the player to move, is linearly decodable at higher accuracy.

\textbf{Choice of $(O,I,T)$.} $T$ is the state of one square, $I$ is the
move sequence, and $O$ is a nested set of move summaries that never
simulate flips: $O_1$ records only the move number and the square's
position; $O_2$ adds whether the square was among the last four moves and
by which player, and its distance to the last move; $O_3$ adds whether the
square has ever been played, by which player (relative to the player to
move), and how many moves ago. $O_3$ is thus the board as it would be if
no disc were ever flipped. The floor measures how much of the board state
is predictable without computing flips; for $O_3$, the headroom is exactly
what the flip dynamics add.

\textbf{Model and probes.} We use the released Othello-GPT checkpoint of \citet{li2022emergent} and
evaluate it on 1{,}000 games of uniformly random legal play, the same
distribution as its training data. The model predicts a legal next move
for 99.94\% of positions (58{,}958/58{,}992; 98.67\% after a pass, which
the input does not mark). Floor and probes share one set of rows: eight
random positions per game and every square at each (512{,}000 rows),
split by game into 60\% fit, 20\% selection and 20\% test. Rows at the
final move (1.6\%) are excluded from the probe because the model reads
at most 59 moves. For each square we fit the MLP and linear decoders of
the decoder bank on the 512-d residual stream after the embedding or after each of the
8 blocks. The layer and each square's decoder are chosen by McFadden
$R^2$ on the selection games, refitted on fit and selection, and scored
once on the test games; 95\% CIs resample test games. The selected layer
is 5 for black/white/empty and 6 for mine/theirs/empty. Probe accuracy
is 0.892 (0.887--0.897) for black/white/empty, with a linear probe at
0.745, and 0.985 (0.980--0.989) for mine/theirs/empty, with a linear
probe at 0.986. Our
probes are trained on about 4{,}700 positions per square, fewer than
in the original work, so their accuracies, and hence $U_H$, are lower
bounds.

\textbf{Result} (Table~\ref{tab:othello}). The floor rises from 0.62
($O_1$) to $\approx 0.80$ ($O_3$). For black/white/empty, our nonlinear probe (0.89)
recovers 72\% of the headroom above $O_1$ but only 45\% above $O_3$: much
of its accuracy is available from disc placement without flips. The
linear black/white probe (0.74) recovers 33\% above $O_1$ and falls
below the $O_3$ floor ($U_H=-0.30$), so its accuracy is explained by
placement alone. For mine/theirs/empty the representation is linear: the
linear probe reaches 0.99 and recovers 92\% of the headroom even above
$O_3$ (96\% above $O_1$), matching the nonlinear probe. The world-model
claim therefore holds most strongly in the relative encoding. The floor depends on which move summaries are allowed;
$O_3$ is the richest summary we tested that still excludes flips. A
floor that approximates flips heuristically would be higher and would
lower $U_H$.

\subsection{Truth directions}
\begin{table}
\centering
\begin{tabular}{llll}
\toprule
Dataset & Decoder & Floor acc. (95\% CI) & Linear acc. \\
\midrule
cities & Linear & 0.46 (0.40--0.52) & 0.46 \\
neg\_cities & Linear & 0.45 (0.39--0.50) & 0.45 \\
sp\_en\_trans & Linear & 0.65 (0.52--0.76) & 0.65 \\
neg\_sp\_en\_trans & GBM & 0.49 (0.38--0.62) & 0.56 \\
larger\_than & Linear & 0.99 (0.97--1.00) & 0.99 \\
smaller\_than & Linear & 0.97 (0.96--0.99) & 0.97 \\
cities\_cities\_conj & Shallow GBM & 0.56 (0.50--0.62) & 0.54 \\
cities\_cities\_disj & Linear & 0.53 (0.48--0.59) & 0.53 \\
companies\_true\_false & Shallow GBM & 0.53 (0.46--0.59) & 0.43 \\
common\_claim\_true\_false & Linear & 0.62 (0.59--0.65) & 0.62 \\
counterfact\_true\_false & GBM & 0.53 (0.52--0.54) & 0.50 \\
\bottomrule
\end{tabular}
\caption{Wording-only floors for the truth datasets.
The floor is the test accuracy of the
selected decoder on word and character $n$-grams of the statement; Linear acc.\ is the linear decoder's. All datasets are balanced, so
chance accuracy is 0.50 (0.51 for sp\_en\_trans and neg\_sp\_en\_trans).
The ceiling is taken as 1, since the truth value is fixed by the
statement and world knowledge.}
\label{tab:truth}
\end{table}

\textbf{Published claim.}
\citet{marks2023geometry} report linear directions in LLM activations that separate true from false statements and transfer across datasets, including to negated statements.

\textbf{Choice of $(O,I,T)$.}
The target $T$ is the binary truth label, $I$ is the full statement, and $O$ is its wording-only representation using $n$-gram features. The floor therefore measures how well truth can be predicted from surface form alone, without access to model activations or pretrained world knowledge. We evaluate on a random split. For transfer experiments, the floor decoder is trained on the same source dataset as the activation probe and evaluated unchanged on the target dataset; we call this the transfer floor.

\textbf{Result.}
Table~\ref{tab:truth} shows that wording alone is close to chance ($0.5$) on most truth datasets. For cross-dataset transfer (Table~\ref{tab:truth_transfer}), we compare the published probes with this transfer floor.

When trained on \texttt{larger\_than} and \texttt{smaller\_than} and tested on \texttt{sp\_en\_trans}, the floor is $0.49$, while the published 13B and 70B probes both reach $0.97$. Similarly, training on \texttt{cities} and \texttt{neg\_cities} and testing on \texttt{neg\_sp\_en\_trans} gives a floor of $0.50$, compared with probe accuracies of $0.96$ and $0.99$. In both cases, the probe transfers substantially better than surface wording alone.

\begin{table}[!htbp]
\centering
\setlength{\tabcolsep}{4pt}
\begin{tabular}{lllll}
\toprule
\makecell{\textcolor{blue}{Train} $\to$ \textcolor{green!60!black}{Test}} &
Decoder &
\makecell{Floor acc.\\(95\% CI)} &
\makecell{Linear\\acc.} &
\makecell{Published probe acc.\\(13B / 70B)} \\
\midrule

\textcolor{blue}{larger\_than + smaller\_than}
$\to$
\textcolor{green!60!black}{sp\_en\_trans}
&
GBM
&
0.49 (0.44--0.55)
&
0.50
&
0.97 / 0.97
\\

\textcolor{blue}{cities + neg\_cities }
$\to$
\textcolor{green!60!black}{neg\_sp\_en\_trans}
&
Shallow GBM
&
0.50 (0.45--0.55)
&
0.50
&
0.96 / 0.99
\\

\textcolor{blue}{cities}
$\to$
\textcolor{green!60!black}{neg\_cities \tablefootnote{The decoder is selected on held-out source rows, which contain unseen city--country pairs, so no decoder is informative there (shallow GBM pseudo-$R^2= 0.035$, linear pseudo-$R^2 = 0.001$); neg\_cities repeats every cities pair with ``not'' inserted and the label flipped, so pairs that shallow GBM memorised during refitting are wrong on the target ($0.35$, below the $0.50$ chance level), while the near-prior linear decoder stays at chance ($0.47$).
}}
&
Shallow GBM
&
0.35 (0.32--0.37)
&
0.47
&
0.75 / 0.37
\\

\textcolor{blue}{larger\_than}
$\to$
\textcolor{green!60!black}{smaller\_than}
&
Linear
&
0.00 (0.00--0.00)
&
0.00
&
0.07 / 0.55
\\

\bottomrule
\end{tabular}

\caption{Wording-only floors for cross-dataset transfer, against the published
mass-mean probe accuracy on Llama-2-13B and Llama-2-70B. Floors are trained on
the source datasets and scored on the target; chance accuracy is 0.50.}

\label{tab:truth_transfer}
\end{table}

The other two transfers give a different pattern. For
\texttt{larger\_than}$\rightarrow$\texttt{smaller\_than}, the floor is $0.00$, while the 13B and 70B probes obtain $0.07$ and $0.55$. For
\texttt{cities}$\rightarrow$\texttt{neg\_cities}, the floor is $0.35$, while the probes obtain $0.75$ and $0.37$. Both floors fall well below chance ($0.50$) because cues learned on the source dataset invert on the target, so headroom normalization is not applicable (Section~\ref{sec:theory}); we compare the probe accuracies directly with the floor and with chance.

\subsection{User attributes in conversations}
\label{app:llm_users}

\textbf{Published claim.} \citet{chen2024designing} train linear probes
on Llama-2-Chat-13B to read a user's age, gender, education, and
socioeconomic status from a conversation, as evidence that the model
builds an internal model of its user. They report two kinds of probe.
\emph{Reading} probes are applied after a prompt that asks the model
about the user, and reach 0.94--0.98. \emph{Control} probes are applied
at the end of the user's message and are used for steering, and reach
0.91--0.96.

\textbf{Choice of $(O,I,T)$.} $T$ is the user's attribute, $I$ is the
conversation, and $O$ is its surface wording: word 1--2-gram and
character 2--4-gram TF-IDF features, with no pretrained model. The
linear decoder uses the sparse features directly, and the other decoders
use a 1{,}024-component truncated SVD of them. The linear decoder wins every
split.

\textbf{Matching the published evaluation.} We use the released synthetic
conversations, in which GPT-3.5 and Llama-2-Chat role-play users with a
given attribute (2{,}900--4{,}500 conversations per attribute), pooled
across both sources as the published probes were trained. We use a
stratified 80/20 split, with 20\% of the training rows held out for
decoder selection. Chance accuracy is
the share of the most common class.

\textbf{Additional evaluations.} To see where the cue comes from, we also
fit the floor on the user's turns only, on each source separately, and
on GPT-3.5 conversations tested on Llama-2-Chat ones.
\begin{table}
\centering
\setlength{\tabcolsep}{4pt}
\begin{tabular}{llllll}
\toprule
Attribute &
Evaluation &
\makecell{Floor acc.\\(95\% CI)} &
\makecell{Chance\\acc.} &
\makecell{Published probe acc.\\(reading / control)} &
\makecell{$U_H$\\(reading / control)} \\
\midrule

\rowcolor{neurocolor}
\multicolumn{6}{l}{\textbf{Age}} \\

& Pooled & 0.98 (0.97--0.99) & 0.25 & 0.98 / 0.96 & 0.11 / $-$0.78 \\
& User turns & 0.98 (0.97--0.99) & 0.25 & -- & -- \\
& GPT-3.5 & 0.98 (0.97--0.99) & 0.25 & -- & -- \\
& Llama-2-Chat & 0.93 (0.89--0.96) & 0.25 & -- & -- \\
& GPT-3.5$\to$Llama-2-Chat & 0.88 (0.86--0.90) & 0.25 & -- & -- \\

\rowcolor{chemcolor}
\multicolumn{6}{l}{\textbf{Gender}} \\

& Pooled & 0.92 (0.89--0.94) & 0.50 & 0.94 / 0.91 & 0.27 / $-$0.09 \\
& User turns & 0.91 (0.88--0.93) & 0.50 & -- & -- \\
& GPT-3.5 & 0.95 (0.93--0.97) & 0.50 & -- & -- \\
& Llama-2-Chat & 0.77 (0.71--0.82) & 0.50 & -- & -- \\
& GPT-3.5$\to$Llama-2-Chat & 0.76 (0.73--0.78) & 0.50 & -- & -- \\

\rowcolor{metacolor}
\multicolumn{6}{l}{\textbf{Education}} \\

& Pooled & 0.97 (0.96--0.98) & 0.33 & 0.96 / 0.93 & $-$0.24 / $-$1.17 \\
& User turns & 0.98 (0.97--0.99) & 0.33 & -- & -- \\
& GPT-3.5 & 0.97 (0.96--0.98) & 0.33 & -- & -- \\

\rowcolor{omicscolor}
\multicolumn{6}{l}{\textbf{Socioeconomic status}} \\

& Pooled & 0.96 (0.95--0.98) & 0.33 & 0.97 / 0.95 & 0.14 / $-$0.43 \\
& User turns & 0.95 (0.93--0.97) & 0.33 & -- & -- \\
& GPT-3.5 & 0.97 (0.95--0.98) & 0.33 & -- & -- \\
& Llama-2-Chat & 0.97 (0.95--0.99) & 0.33 & -- & -- \\
& GPT-3.5$\to$Llama-2-Chat & 0.90 (0.87--0.92) & 0.33 & -- & -- \\

\bottomrule
\end{tabular}
\caption{Wording-only floors for user attributes. \emph{Pooled}: both
role-play sources, as the published probes were trained.
\emph{User turns}: the user's turns only. \emph{GPT-3.5},
\emph{Llama-2-Chat}: one source. \emph{GPT-3.5 $\to$ Llama-2-Chat}: fitted on
GPT-3.5 conversations and tested on Llama-2-Chat ones. The released data contain no Llama-2-Chat conversations for education, so those rows are absent. The floor is the
test accuracy of the selected decoder (linear in every row). $U_H$
uses a ceiling of 1 and is unstable here because
the headroom is small.}
\label{tab:user}
\end{table}
\textbf{Result} (Table~\ref{tab:user}). Wording alone predicts every
attribute almost as well as the published probes. Each published reading
probe lies within the floor's 95\% interval: 0.98 against a
floor of 0.98 for age, 0.94 against 0.92 for gender, 0.96 against 0.97
for education, and 0.97 against 0.96 for socioeconomic status. Every
control probe falls below the floor. The published accuracies are
therefore largely explained by cues in the wording.

The floor from the user's turns alone (0.91--0.98) matches the floor
from the full conversation, so the cues are in what the role-played user
writes. Trained on GPT-3.5 conversations and tested on Llama-2-Chat ones,
the floor falls to 0.76--0.90, so part of the cue is the generating
model's style rather than the attribute itself.

\textbf{Reading $U_H$ near a ceiling.} With floors of 0.92--0.98, the
headroom is only 0.02--0.08, so a difference of a few points between
probe and floor becomes a large $U_H$. The reading probes give $U_H$ of
0.11, 0.27, and 0.14 for age, gender, and socioeconomic status, and
$-0.24$ for education; the control probes give $-0.09$ to $-1.17$. These
values mainly reflect how small the headroom is, so we read this
analysis by the raw comparison above rather than by $U_H$.

\textbf{Caveats.} The conversations were written to reflect each
attribute, so the label is visible in the text by design. This shows
that the published accuracies cannot establish a user model, not that
the model lacks one; testing that would need data where the attribute
is not cued in the wording.


\section{Conditional probing and predictive $\mathcal V$-information}
\label{app:conditional}

\citet{hewitt2021conditional},
based on predictive $\mathcal V$-information \citep{xu2020theory}, ask how much information about a linguistic property, such as part of speech,
a contextual layer of a language model (ELMo, RoBERTa) contains beyond what
the non-contextual word embeddings already provide. The baseline $B$ is
therefore the embedding layer, and the quantity of interest is the
information that a deeper layer $H$ adds to it, $
I_{\mathcal V}(H \to Y \mid B) = H_{\mathcal V}(Y \mid B) - H_{\mathcal V}(Y \mid B, H),$
estimated by training one probe on $B$ and one on the concatenation $[B;H]$.
They contrast this with the earlier practice of comparing a probe on $H$ with
a probe on $B$ separately. That standalone comparison penalizes $H$ for any
information about word identity it has discarded, even when it adds new
contextual information. Under conditional probing, deeper layers turn out to
carry more information beyond the embeddings than standalone comparisons had
suggested.

Our floor  plays the role of the baseline defined by them, with two differences: $O$ is a
declared information set chosen by the scientific question rather than
another layer of the model, and we pair it with a ceiling $\rho^2_I$, the
predictability supported by the full input.



\subsection{Comparing probe normalization with conditional probes}
\label{app:conditional-relations}

Here, $T$ is the target, $O$ the observed information that defines the floor $\rho_O^2$, $I$ the full input, $H$ the representation
and $B$ a baseline (here $B=O$). $R^2_X$ is the held-out $R^2$ of a probe
on $X$, and $[B;H]$ is the concatenation of $B$ and $H$. The normalized score, $U_H$ has the form of a skill score. In our method, two reference points, floor and ceiling are what $O$ and $I$ themselves determine about $T$, not a fixed baseline and a perfect score.

\paragraph{For continuous targets, conditional probing reduces to a
difference in $R^2$.}
$\mathcal V$-information measures how much better a probe predicts $T$ when it
is given $X$ than when it is given nothing, with prediction quality scored by
log-loss. Suppose the probe outputs a Gaussian with mean $\varphi(X)$ and fixed
variance $\tfrac12$. Its log-loss is then
$-\log p(T\mid X) = (T-\varphi(X))^2 + \text{const}$, so minimizing log-loss
is the same as minimizing squared error, and the best achievable log-loss is
the probe's MSE plus a constant. With no input, the best prediction is the
mean of $T$, whose MSE is $\mathrm{Var}(T)$. Subtracting,
\[
I_{\mathcal V}(X\to T) = \mathrm{Var}(T) - \mathrm{MSE}_X
= \mathrm{Var}(T)\,R^2_X .
\]
The same subtraction with baseline $B$ gives
\[
\frac{I_{\mathcal V}(H\to T\mid B)}{\mathrm{Var}(T)}
= \frac{\mathrm{MSE}_B - \mathrm{MSE}_{B,H}}{\mathrm{Var}(T)}
= R^2_{B,H}-R^2_B .
\]
For continuous targets, conditional probing is therefore just the gain in
$R^2$ from adding $H$ to $B$, so we report it in $R^2$ units for the
meta-analysis. For Othello the targets are categorical, so we report the
corresponding gain in accuracy.

\paragraph{$U_{H\mid O}$ is conditional probing with $B=O$ divided by the headroom.}

The normalized  score $U_H$ compares the probe on $H$ with the floor directly. We define a conditional
variant, $U_{H|O}$ that credits $H$ only for what it adds to $O$:
\[
U_{H|O} := \frac{R^2_{O,H}-\rho^2_O}{\rho^2_I-\rho^2_O},
\]
where $R^2_{O,H}$ is the held-out $R^2$ of a probe given both $O$ and $H$.
Its numerator is conditional probing with $B=O$, so
$U_{H|O}$ is conditional probing normalized by the headroom.

\paragraph{In the population, $U_H\le U_{H\mid O}\le 1$.}
The two scores share the denominator $\rho_I^2-\rho_O^2$ and differ only in
whether the probe sees $H$ alone or $H$ together with $O$. With Bayes-optimal
predictors, more information never lowers $R^2$. Since $(O,H)$ contains $H$,
$R^2_{O,H}\ge R^2_H$, which gives $U_{H\mid O}\ge U_H$. Since $(O,H)$ is a
function of $I$ when $H=\Phi(I)$, $R^2_{O,H}\le\rho_I^2$, which gives
$U_{H\mid O}\le 1$. The gap
\[
U_{H\mid O}-U_H=\frac{R^2_{O,H}-R^2_H}{\rho_I^2-\rho_O^2}
\]
is the information about $T$ that $O$ carries and $H$ has discarded. It is
zero when $H$ retains everything in $O$ that is relevant to $T$. In the
population, $U_H$ is therefore conservative: it can understate, but not overstate, what $H$ adds beyond $O$.
With finite samples, estimated reference points and a fixed decoder family, either inequality can fail;
out of distribution in Table~\ref{tab:meta_cond}, for example, $U_H$ slightly exceeds $U_{H\mid O}$.

\subsection{Meta-analysis.}
$H$ is the last layer of the trained xl transformer at the final position. $O$ is the summary statistics of $\sigma^2$ (Appendix~\ref{app:endpoint-features}). Every probe (for $O$, $H$, $[O;H]$) is the linear decoder from the decoder bank of
Table~\ref{tab:canonical-probe-bank}.

\begin{table}[!htbp]
\centering
\begin{tabular}{lllll}
\toprule
Condition & \multicolumn{2}{c}{ID} & \multicolumn{2}{c}{OOD} \\
Target & $\log S_1$ & $\tau^2$ & $\log S_1$ & $\tau^2$ \\
\midrule
Estimated floor $\hat\rho_O^2$ & $0.01 \pm 0.01$ & $-0.00 \pm 0.00$ & $0.09 \pm 0.01$ & $-0.00 \pm 0.00$ \\
Estimated ceiling $\hat\rho_I^2$ & $0.63 \pm 0.02$ & $0.66 \pm 0.01$ & $0.41 \pm 0.02$ & $0.36 \pm 0.03$ \\
Estimated headroom $\hat\Delta=\hat\rho_I^2-\hat\rho_O^2$ & $0.62 \pm 0.01$ & $0.66 \pm 0.01$ & $0.33 \pm 0.02$ & $0.36 \pm 0.03$ \\
Hidden probe $R^2_H$ & $0.60 \pm 0.04$ & $0.64 \pm 0.02$ & $0.40 \pm 0.02$ & $0.34 \pm 0.02$ \\
Probe on $[O;H]$, $R^2_{O,H}$ & $0.59 \pm 0.04$ & $0.64 \pm 0.02$ & $0.40 \pm 0.02$ & $0.34 \pm 0.03$ \\
Normalized score $ U_H$ & $0.95 \pm 0.04$ & $0.97 \pm 0.04$ & $0.96 \pm 0.02$ & $0.97 \pm 0.02$ \\
Conditional normalized score $ U_{H\mid O}$ & $0.95 \pm 0.05$ & $0.97 \pm 0.04$ & $0.95 \pm 0.02$ & $0.95 \pm 0.03$ \\
Conditional probe ($B=O$), $R^2_{O,H}-\hat\rho_O^2$ & $0.58 \pm 0.04$ & $0.64 \pm 0.03$ & $0.31 \pm 0.02$ & $0.34 \pm 0.02$ \\
\bottomrule
\end{tabular}
\caption{Meta-analysis transformer (xl, last layer, final position): normalized headroom and conditional probing with baseline $B=O$. Every probe, floor and ceiling is the \textbf{linear decoder}; the floor and ceiling are estimated from $O$ and $I$. Mean $\pm$ SD over five seeds.}
\label{tab:meta_cond}
\end{table}

We found that conditional normalized score is roughly equal to normalized score (Table~\ref{tab:meta_cond}). For both  $\tau^2$ and $\log S_1$, $U_H$ and $U_{H\mid O}$ differ by at most $0.02$ in every condition. Raw conditional probe gains mix up the representation and the data and normalizing by headroom separates them.
From in distribution to out of distribution, Conditional Probe $(H\mid B=O)$ halves from $0.58\to0.31$ for $\log S_1$ and
$0.64\to0.34$ for $\tau^2$. Alone, this indicates that the representation got twice as bad. But the estimated headroom $\hat\rho_I^2-\hat\rho_O^2$ also shrinks ($0.62\to0.33$ for $\log S_1$; $0.66\to0.36$ for $\tau^2$), and the normalized score
does not fall: $U_{H\mid O}=0.95\to0.95$ for $\log S_1$ and $0.97\to0.95$ for $\tau^2$. The drop therefore reflects information missing from the shifted data, not a worse representation. Conditional probing cannot differentiate without a ceiling score.
\begin{table}[t!htbp]
\centering
\begin{tabular}{lll}
\toprule
Target & Black/White/Empty & Mine/Theirs/Empty \\
\midrule
Layer & 5 & 6 \\
Estimated floor (rung $O_3$) & 0.762 (0.754--0.768) & 0.803 (0.797--0.808) \\
Ceiling (exact) & 1.000 & 1.000 \\
Headroom & 0.238 & 0.197 \\
Probe on $O$ & 0.768 (0.762--0.774) & 0.808 (0.803--0.814) \\
Hidden probe & 0.745 (0.738--0.752) & 0.986 (0.981--0.990) \\
Probe on $[O;H]$ & 0.767 (0.760--0.773) & 0.989 (0.984--0.992) \\
Normalized score $U_H$ & -0.069 & 0.929 \\
Conditional normalized score $U_{H\mid O}$ & 0.021 & 0.942 \\
Conditional probe ($B=O$), $\mathrm{Acc}_{O,H}-\mathrm{Acc}_O$ & -0.001 & 0.180 \\
\bottomrule
\end{tabular}
\caption{Othello-GPT: normalized headroom and conditional probing with baseline $B=O$ (rung $O_3$: disc placement without flips). Test accuracy of the linear decoder, one probe per square; the ceiling is exactly 1 because the board is a function of the moves. The estimated floor is the linear decoder of the floor ladder (Table~\ref{tab:othello}), fitted once across all squares; the probe on $O$ fits one linear probe per square on the same features and on the probe's rows, which exclude the final move, so the two differ slightly. The conditional probe is measured from the probe on $O$; $U_H$ and $U_{H\mid O}$ are measured from the estimated floor.}
\label{tab:appK-othello}
\end{table}

\subsection{Othello-GPT.} We fit one linear probe per square on the residual stream
(block 5 for black/white/empty, block 6 for mine/theirs/empty; 1,000 games split
60/20/20), with $O$ the move-sequence features of rung $O_3$ (disc placement without
flips) and a ceiling of 1, since the board is a function of the moves. Because
every probe here, including the floor, is linear, the floor (0.76 and 0.80) is slightly
lower than the decoder-bank floor of Appendix~\ref{app:llm}. The standalone and
conditional scores reach the same verdict (Table~\ref{tab:appK-othello}). For
black/white/empty, $H$ adds nothing beyond disc placement: the conditional gain is
$-0.001$, $U_H=-0.07$ and $U_{H|O}=0.02$. For mine/theirs/empty, $H$ captures most of
what the flips add: the conditional gain is $0.18$, $U_H=0.93$ and $U_{H|O}=0.94$.
The normalized and conditional normalized scores agree in both cases. Therefore, both methods support the conclusion of
Section~\ref{sec:llm_probes}: a linear readout of the absolute board is explained by
disc placement, while with mine/theirs labels a linear probe recovers the board state that requires simulating flips.

\newpage
\section{Additional figures and tables}
\FloatBarrier
\begin{table}[!htbp]
\centering
\begin{tabular}{llllllll}
\toprule
 & \multicolumn{2}{c}{$\log S_1$ $R^2$} & \multicolumn{2}{c}{$\tau^2$ $R^2$} & \multicolumn{2}{c}{$\mathrm{MSE}_{\rm model}/\mathrm{MSE}_{\rm oracle}$} & MSE \\
 & ID & OOD & ID & OOD & ID & OOD & OOD/ID \\
\midrule
xs & 0.28 $\pm$ 0.20 & 0.30 $\pm$ 0.04 & 0.22 $\pm$ 0.20 & 0.24 $\pm$ 0.03 & 1.05 $\pm$ 0.01 & 5.05 $\pm$ 2.25 & 28.3 \\
s & 0.60 $\pm$ 0.04 & 0.29 $\pm$ 0.03 & 0.54 $\pm$ 0.03 & 0.26 $\pm$ 0.02 & 1.02 $\pm$ 0.00 & 2.44 $\pm$ 0.37 & 14.2 \\
m & 0.60 $\pm$ 0.03 & 0.34 $\pm$ 0.03 & 0.60 $\pm$ 0.04 & 0.28 $\pm$ 0.03 & 1.03 $\pm$ 0.00 & 2.58 $\pm$ 0.73 & 14.9 \\
l & 0.61 $\pm$ 0.02 & 0.35 $\pm$ 0.05 & 0.58 $\pm$ 0.05 & 0.29 $\pm$ 0.05 & 1.03 $\pm$ 0.01 & 2.23 $\pm$ 0.59 & 12.9 \\
xl & 0.64 $\pm$ 0.05 & 0.40 $\pm$ 0.02 & 0.64 $\pm$ 0.02 & 0.34 $\pm$ 0.02 & 1.03 $\pm$ 0.00 & 2.08 $\pm$ 0.51 & 12.0 \\
\bottomrule
\end{tabular}
\caption{In-distribution (ID) vs.\ out-of-distribution (OOD) probe $R^2$ and prediction MSE at the final checkpoint (mean $\pm$ sample standard deviation across five seeds). MSE is measured at the final position; the oracle is the conditional-Bayes estimator $S_2/(S_1+1)$, which uses the true $\tau^2$.}
\label{tab:dissociation}
\end{table}
\begin{table*}[!htbp]
\centering
\small
\setlength{\tabcolsep}{5pt}
\renewcommand{\arraystretch}{1.08}

\begin{tabularx}{\textwidth}{Xccccc}
\toprule
Target &
Floor &
Ceiling &
Headroom &
\makecell{Representation\\probe} &
$U_H$ \\
\midrule

\rowcolor{neurocolor}
\multicolumn{6}{l}{\textbf{Neuroscience}} \\

Force $x$
& $0.11 \pm 0.02$
& $0.99 \pm 0.01$
& $0.88 \pm 0.03$
& $0.75 \pm 0.01$
& $0.72 \pm 0.02$ \\

Force $y$
& $0.09 \pm 0.02$
& $1.00 \pm 0.00$
& $0.91 \pm 0.02$
& $0.57 \pm 0.02$
& $0.53 \pm 0.03$ \\

\addlinespace
\rowcolor{chemcolor}
\multicolumn{6}{l}{\textbf{Chemistry}} \\

HOMO
& $0.39 \pm 0.00$
& $0.63 \pm 0.00$
& $0.24 \pm 0.00$
& $0.56 \pm 0.01$
& $0.69 \pm 0.07$ \\

LUMO
& $0.50 \pm 0.02$
& $0.69 \pm 0.01$
& $0.19 \pm 0.01$
& $0.66 \pm 0.01$
& $0.82 \pm 0.09$ \\

\addlinespace
\rowcolor{metacolor}
\multicolumn{6}{l}{\textbf{Meta-analysis}} \\

$\log S_1$ (ID)
& $0.01 \pm 0.00$
& $0.82 \pm 0.01$
& $0.81 \pm 0.02$
& $0.64 \pm 0.05$
& $0.78 \pm 0.06$ \\

$\tau^2$ (ID)
& $0.00 \pm 0.00$
& $0.76 \pm 0.01$
& $0.76 \pm 0.01$
& $0.64 \pm 0.02$
& $0.84 \pm 0.04$ \\

$\log S_1$ (OOD)
& $0.09 \pm 0.01$
& $0.47 \pm 0.02$
& $0.37 \pm 0.02$
& $0.40 \pm 0.02$
& $0.82 \pm 0.08$ \\

$\tau^2$ (OOD)
& $0.00 \pm 0.00$
& $0.40 \pm 0.02$
& $0.40 \pm 0.02$
& $0.34 \pm 0.02$
& $0.85 \pm 0.08$ \\

\addlinespace
\rowcolor{omicscolor}
\multicolumn{6}{l}{\textbf{Single-cell}} \\

Synthetic
& $-0.02 \pm 0.02$
& $0.87 \pm 0.01$
& $0.89 \pm 0.02$
& $0.60 \pm 0.03$
& $0.70 \pm 0.03$ \\

10k human
& $0.60 \pm 0.02$
& $0.97 \pm 0.01$
& $0.38 \pm 0.02$
& $0.81 \pm 0.02$
& $0.57 \pm 0.03$ \\

20k donors (pooled)
& $0.13 \pm 0.04$
& $0.97 \pm 0.00$
& $0.84 \pm 0.04$
& $0.57 \pm 0.03$
& $0.52 \pm 0.03$ \\

40k breast
& $0.15 \pm 0.06$
& $0.87 \pm 0.02$
& $0.72 \pm 0.04$
& $0.35 \pm 0.04$
& $0.28 \pm 0.05$ \\

40k lung
& $0.06 \pm 0.06$
& $0.86 \pm 0.02$
& $0.80 \pm 0.05$
& $0.44 \pm 0.05$
& $0.47 \pm 0.05$ \\

\bottomrule
\end{tabularx}
\caption{Headroom-normalized decodability across domains
(mean $\pm$ sample standard deviation across five seeds).
$U_H=(R_H^2-\rho_O^2)/(\rho_I^2-\rho_O^2)$ is the fraction of
headroom recovered by the representation probe.
Neuroscience uses the $O_4$ pre-bump floor and a decoder-bank probe on Area~2 response counts, fitted within each recording session; meta-analysis uses the exact reference points of
Table~\ref{tab:exact-endpoints} with final-checkpoint xl last-layer probes;
single-cell data use reference points estimated at full sequencing depth and a probe on frozen scGPT representations (64 PCA components); chemistry uses the $O_4$ physicochemical floor and a probe on the LoRA-fine-tuned TabPFN.}

\label{tab:cross-domain-calibration-legacy}
\end{table*}

\end{document}